\documentclass{article}

\usepackage{microtype}
\usepackage{graphicx}
\usepackage{subcaption}
\usepackage{booktabs} 

\usepackage{hyperref}

\usepackage[preprint]{icml2026}

\usepackage{amsmath}
\usepackage{amssymb}
\usepackage{mathtools}
\usepackage{amsthm}

\usepackage[capitalize,noabbrev]{cleveref}

\newtheorem{fact}{Fact}
\newtheorem{pathology}{Pathology}
\newtheorem{prop}{Proposition}
\newcommand{\real}{\mathbb{R}}
\newcommand{\cA}{\mathcal{A}}

\newcommand{\cS}{\mathcal{S}}

\newcommand{\cM}{\mathcal{M}}

\newcommand{\cP}{P}

\newcommand{\cX}{\mathcal{X}}

\newcommand*{\infn}[1]{\left\|{#1}\right\|_{\infty}}

\newcommand{\smallplus}{\raisebox{.45ex}{$\mkern1mu\scriptscriptstyle+$}}
\newcommand{\smallminus}{\raisebox{.45ex}{$\mkern1mu\scriptscriptstyle-$}}
\newcommand{\expec}{\mathbb{E}}

\DeclareMathOperator{\dom}{dom}
\DeclareMathOperator{\rint}{rint}
\DeclareMathOperator{\argmin}{argmin}
\theoremstyle{plain}
\newtheorem{theorem}{Theorem}[section]

\newtheorem{lemma}[theorem]{Lemma}
\newtheorem{corollary}[theorem]{Corollary}
\theoremstyle{definition}
\newtheorem{definition}[theorem]{Definition}
\newtheorem{assumption}[theorem]{Assumption}
\theoremstyle{remark}

\usepackage[textsize=tiny]{todonotes}

\icmltitlerunning{Why Policy Gradient Algorithms Work for Undiscounted Total-Reward MDPs}

\begin{document}

\twocolumn[
  \icmltitle{Why Policy Gradient Algorithms Work for Undiscounted Total-Reward MDPs}



  \icmlsetsymbol{equal}{*}

  \begin{icmlauthorlist}
    \icmlauthor{Jongmin Lee}{j}
  \icmlauthor{Ernest Ryu}{e}
  \end{icmlauthorlist}

  \icmlaffiliation{j}{Seoul National University, Department of Mathematical Sciences}
  \icmlaffiliation{e}{UCLA, Department of Mathematics}


  \icmlkeywords{Machine Learning, ICML}

  \vskip 0.3in
]



\printAffiliationsAndNotice{}  

\begin{abstract}
The classical policy gradient method is the theoretical and conceptual foundation of modern policy-based reinforcement learning (RL) algorithms. Most rigorous analyses of such methods, particularly those establishing convergence guarantees, assume a discount factor $\gamma < 1$. In contrast, however, a recent line of work on policy-based RL for large language models uses the undiscounted total-reward setting with $\gamma = 1$, rendering much of the existing theory inapplicable. In this paper, we provide analyses of the policy gradient method for undiscounted expected total-reward infinite-horizon MDPs based on two key insights: (i) the classification of the MDP states into recurrent and transient states is invariant over the set of policies that assign strictly positive probability to every action (as is typical in deep RL models employing a softmax output layer) and (ii) the classical state visitation measure (which may be ill-defined when $\gamma = 1$) can be replaced with a new object that we call the transient visitation measure.
\end{abstract}

\section{Introduction}

Since the seminal Policy Gradient Theorem \citep{sutton1999policy}, policy gradient algorithms have been a cornerstone of modern reinforcement learning (RL). Unlike classical dynamic programming approaches, policy gradient methods directly optimize the policy using the gradient of the expected total-reward. These methods, along with their deep learning variants, have achieved remarkable practical success, and their convergence properties have been extensively studied for Markov decision processes (MDP) with discount factor $\gamma < 1$.

More recently, however, a large body of work has emerged on training large language models within the RL framework without discounting ($\gamma = 1$) and arbitrarily long horizons, as in reinforcement learning from human feedback (RLHF) \citep{christiano2017deep} and reinforcement learning with verifiable rewards (RLVR) \citep{guo2025deepseek}. Yet, the convergence of policy gradient methods in this undiscounted total-reward setup remains largely unexplored, and even the policy gradient theorem itself has not been rigorously established in this setup.




\paragraph{Contribution.}
In this work, we study the convergence of policy gradient methods for undiscounted expected total-reward infinite-horizon MDPs. Our analysis is based on two key insights: (i) the classification of the MDP states into recurrent and transient states is invariant over the set of policies that assign strictly positive probability to every action (as is typical in deep RL models employing a softmax output layer) and (ii) the classical state visitation measure (which may be ill-defined when $\gamma=1$) can be replaced with a new object that we call the \emph{transient visitation measure}. Leveraging these insights, we establish convergence guarantees for projected policy gradient and natural policy gradient algorithms in the tabular setting. We also extend the analysis to the generative model setting and derive sample complexity for stochastic natural policy gradient.


\subsection{Related works}

\paragraph{Undiscounted total-reward infinite-horizon MDP.}
The setup of undiscounted total-reward MDP was first introduced by \cite{savage1965gamble}. For the well-definedness of the value function, \cite{schal1983stationary} considered the finiteness of $V^\pi_+$ and $V^\pi_-$ (defined in the next section) and the existence of an optimal policy was first proved by \cite{van1981stochastic}. With different additional assumptions, three models of total-reward setup have been proposed and studied: stochastic shortest path model \citep{Eaton1962OptimalPS}, positive model \citep{blackwell1967positive}, and negative model \citep{strauch1966negative}. The stochastic shortest path model assumes single absorbing terminal state and the existence of a policy that reaches the terminal state with probability $1$ from any initial state \citep{bertsekas1991analysis}.
The positive model assumed $V_+^\pi$ is finite, and for each $s$, there exist $a$ with $r(s,a) \ge 0$ \citep[Section 7.2]{Puterman2014}. The negative model assumed $V_+^\pi=0$ and there exist $\pi $ for which $V_-^\pi(s)> \infty$ for all $s$ \citep[Section 7.3]{Puterman2014}. In this work, the undiscounted total-reward model with the assumption that $V^\pi$ is finite, motivated by modern reinforcement learning frameworks and needed to establish the transient policy gradient, does not fall into previous categories. 

\paragraph{Policy gradient method.}
Policy gradient methods \citep{williams1992simple,sutton1999policy,konda1999actor,kakade2001natural} are foundational reinforcement learning algorithms, commonly implemented with deep neural networks for policy parameterization \citep{schulman2015trust,schulman2017proximal}. In line with their practical success, convergence and sample complexity of policy gradient variants have been extensively studied across settings. In discounted total-reward infinite-horizon MDP,  \cite{agarwal2021theory, xiao2022convergence, bhandari2024global,mei2020global} analyzed convergence of projected policy gradient and naive policy gradient with softmax parametrization. The natural policy gradient, introduced by \cite{kakade2001natural} and viewable as a special case of mirror descent \citep{shani2020adaptive}, has been analyzed by \cite{agarwal2021theory, cen2022fast,xiao2022convergence, lan2023policy}. In the average reward MDP, where the goal is to maximize long-term, steady-state performance,
convergence and sample complexity results have been established by \cite{even2009online,murthy2023convergence,bai2024regret,kumar2024global, li2025stochastic}, and related analyses exist for the finite horizon setup as well \citep{hambly2021policy,guo2022theoretical,klein2023beyond}. 

In undiscounted total-reward infinite-horizon MDP, however, there are few results on policy gradient methods. Since \citet{sutton1999policy} established the policy gradient theorem only for the discounted total-reward and average reward MDPs, \citet{bojun2020steady,ribera2025reinforcement} analyze policy gradients for the undiscounted total-reward random time horizon MDP. 
Specifically, \citet{bojun2020steady} considers an episodic learning process that can be viewed as an ergodic Markov chain with finite episode length, and establish a policy gradient theorem via the steady state distribution. \citet{ribera2025reinforcement} consider trajectory dependent random termination times and prove a policy gradient theorem with an almost surely finite termination time. We note that neither work further analyzes the convergence of policy gradient methods, and their setups and assumptions differ from ours, as shown in the next section.

\section{Undiscounted expected total-reward infinite-horizon MDPs}\label{sec:2}
In this work, we consider undiscounted total-reward infinite-horizon Markov decision processes (MDPs). We review basic definitions and assumptions of undiscounted MDPs and reinforcement learning (RL). For further details, we refer the readers to references such as \citep[Section 7]{Puterman2014} or \citep{sutton2018reinforcement}.

\paragraph{Undiscounted Markov decision processes.}
Let $\cM(\cX)$ be the space of probability distributions over a set $\cX$. Write $(\cS, \cA, P, r, \mu)$ to denote the infinite-horizon undiscounted MDP with finite state space $\cS$, finite action space $\cA$, transition matrix $P\colon \cS \times \cA \rightarrow \cM(\cS)$, bounded reward $r\colon  \cS \times \cA \rightarrow [-R,R]$ with some $R<\infty$, and initial state distribution $\mu \in \cM(\cS)$. We say the reward is nonnegative if $r(s,a) \ge 0$ for all $ s \in \cS,a \in \cA$. Denote $\pi\colon \cS \rightarrow \cM(\cA)$ for a policy. 
Define
\begin{alignat*}{3}
&\Pi &&= \text{set of all policies}=\cM(\cA)^\cS,\\    
&\Pi_{+} &&= \{ \pi\in \Pi \,|\, \pi (a \,|\, s) >0 \text{ for all } s,a\}.
\end{alignat*}
So, $\Pi_+$ is the (relative) interior of $\Pi$. 

Let
\begin{align*}
V_{\smallplus}^\pi(s) &= \lim_{T \rightarrow \infty}\expec_\pi\bigg[\sum^{T-1}_{i=0} \max\{r(s_i,a_i),0\}  \,\Big|\, s_{0}=s\bigg]\\
V_{\smallminus}^\pi(s) &= \lim_{T \rightarrow \infty}\expec_\pi\bigg[\sum^{T-1}_{i=0}\max\{- r(s_i,a_i),0\}  \,\Big|\, s_{0}=s\bigg],
\end{align*}
where $\expec_{\pi}$ denotes expectation over trajectories 
$(s_0, a_0, s_1, a_1, \dots, s_{T-1}, a_{T-1}) $ 
induced by the policy $\pi$.  
Since both summands are nonnegative, the monotone convergence theorem guarantees that each limit exists (possibly infinite). To ensure the well-definedness of the value function $V^\pi$, we impose the following assumption.  
\begin{assumption}[Finiteness of value function]\label{assump_total}
$V_+^\pi(s)<\infty,\,V_-^\pi(s)<\infty$ for all $\pi \in \Pi$ and $s\in \cS$.
\end{assumption}


While bounded rewards ensure the well-definedness of the value function and policy gradient in discounted and average-reward settings, our undiscounted total-reward setup requires Assumption~\ref{assump_total}. As noted, in recent training of large language models under the frameworks of reinforcement learning with human feedback (RLHF) and reinforcement learning with verifiable reward (RLVR), finite rewards are assigned once at the end of the trajectory, and Assumption~\ref{assump_total} naturally holds.

Under Assumption~\ref{assump_total}, we define the state value function as 
\[V^\pi = V_{\smallplus}^\pi - V_{\smallminus}^\pi,\]
and $V^\pi_\mu = \expec_{s \sim \mu} [V^{\pi}(s)]$ where $\mu \in \cM(\cS)$ is the initial state distribution. Likewise, define
\[
Q_{\smallplus}^\pi(s,a) =\lim_{T \rightarrow \infty}\expec_\pi\bigg[\sum^{T-1}_{i=0} \max\{r(s_i,a_i),0\}  \,\Big|\, s_{0}, a_{0}\!=\!s, a\bigg]\!, 
\]
and $Q_{\smallminus}^\pi$ analogously. Likewise, define  $Q^\pi= Q_{\smallplus}^\pi - Q_{\smallminus}^\pi$. Then, $Q^\pi$ is well-defined under Assumption~\ref{assump_total}, and $Q^\pi = \cP V^\pi +r$ where $P \in \real^{|\cS|\!|\cA| \!\times \!|\cS|} $, and $V^\pi(s) = \mathbb{E}_{a \sim \pi(\cdot\,|\,s) }[Q^\pi(s,a)]$  by definition.

We say $V^{\star}$ is optimal value function if $V^{\star}(s)=\max_{\pi}V^{\pi}(s)$ for all $s \in \cS$ and $\pi$ is an $\epsilon$-optimal policy if $\infn{V^{\star}-V^{\pi}} \le \epsilon$. It is known that the optimal value function and an optimal policy always exist in the undiscounted total-reward setup with finite state and action spaces \citep[Theorem 7.1.9]{Puterman2014}.
(As a technical detail, Theorem~7.1.9 of \cite{{Puterman2014}} is stated in terms of the set of all history-dependent policies, but the proof also works for the set of all stationary policies, which is our focus in this work.)

 For notational conciseness, we write $r^{\pi}(s)=\mathbb{E}_{a \sim \pi(\cdot\,|\,s) }\left[r(s,a)\right]$ for the reward induced by policy $\pi$ and $\cP^\pi(s,s')$ defined as 
 \[\cP^{\pi}(s, s')=
\mathrm{Prob}(s\rightarrow s'\,|\,
a \sim \pi(\cdot\,|\,s), s'\sim P(\cdot\,|\,s,a))\]
is the transition probability induced by policy $\pi$.
Then, we can write
$V^\pi =\sum^\infty_{n=0} (P^\pi)^n r^\pi $. 


In Section~\ref{s:pathology}, we discuss the continuity of the mapping $\pi\mapsto V^\pi$. Since $|\cS|$ and $|\cA|$ are finite, we can identify $\pi$ and $V^\pi$ as finite-dimensional vectors, namely, as $\pi\in \mathbb{R}^{|\cS|\times |\cA|}$ and $V^\pi\in \mathbb{R}^{|\cS|}$.
Therefore, continuity of $\pi\mapsto V^\pi$ can be interpreted as continuity of the mapping from $\mathbb{R}^{|\cS|\times |\cA|}$ to $\mathbb{R}^{|\cS|}$ under the usual metric.

\section{Troubles with undiscounted total-reward MDPs}
\label{s:pathology}






\begin{figure}
    \centering
    \begin{tikzpicture}[>=stealth,shorten >=1pt,auto,node distance=1cm]
  \tikzset{state/.style={circle,draw,minimum size=18pt,inner sep=2pt}}
  \node[state]                 (s1) {$s_1$};
  \node[state]         [right=of s1] (s0) {$s_0$};
  \node[state]        [right=of s0] (s2) {$s_2$};
  \node[state] [left=of s1] (s3) {$s_3$};
  \node[state] [right=of s2] (s4) {$s_4$};

  \path[->]
    (s0) edge              node[above]  {$0$} (s1)
    (s0) edge              node[above]  {$0$} (s2)
    (s1) edge              node[above]  {$-1$} (s3)
    (s2) edge              node[above] {$1$} (s4)
    (s3) edge[loop above]  node        {0} ()
    (s4) edge[loop above]  node        {0} ()
    (s1) edge[loop above]  node        {0} ()
    (s2) edge[loop above]  node        {0} ();
\end{tikzpicture}
    \caption{Pathological MDP: The value function is discontinuous at the optimal policy, and the optimal action-value function does not specify the optimal policy
    }
    \label{fig:example}
\end{figure}
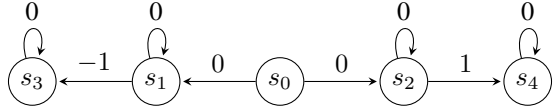





In this section, we point out two pathologies that can arise in total-reward MDPs that do not arise in the discounted setup.


\begin{pathology}\label{discontinuity}
The value function $V^\pi$ may be a discontinuous function of $\pi\mathrel \in \Pi$, even when $|\mathcal{S}|$ and $|\mathcal{A}|$ are finite and $V^\pi$ is finite.
\end{pathology}



In the example of Figure~\ref{fig:example}, the optimal action at state $s_1$ is to remain at $s_1$. Under the optimal policy, taking this optimal action, we have $V^\star(s_1)=0$, but any policy assigning a non-zero probability to the other action, transitioning to $s_3$, yields $V^{\pi}(s_1)=-1$. This example illustrates that the value function can be discontinuous in $\pi$, and a policy gradient method cannot be expected to succeed in the presence of such discontinuities, while the discounted-reward setup guarantees differentiability of the value function. 
We provide the framework to address this pathology in Section~\ref{sec:recurrent_transeint}.



\begin{pathology}\label{optimality}
The optimal action-value function $Q^\star$ does not, by itself, specify the optimal policy $\pi^\star$. In particular, a policy $\pi$ satisfying $\pi(s)\in \arg\max_{a\in \mathcal{A}} Q^\star(s,a)$ for all $s\in \mathcal{S}$ may not be optimal.
\end{pathology}

Again, in the example of Figure \ref{fig:example}, the optimal action at $s_2$ is to transition to $s_4$. However, the non-optimal policy $\pi$ that stays at $s_2$ with probability one also satisfies $Q^\star(s_2,\pi(s_2))=V^\star(s_2)=+1$. In other words, the policy $\pi(s)\in \arg\max_{a\in \mathcal{A}} Q^\star(s,a)$ can be non-optimal, and it is known that additional conditions are needed to specify an optimal policy in this setup \citep[Theorem 7.25]{Puterman2014}. This example illustrates that value-based methods such as value iteration or Q-learning may fail to provide an optimal policy even if they approximate the optimal $Q$-function well. In this work, we study the policy gradient method, a policy-based RL method, and show that it is not subject to this issue.





\section{Recurrent-transient theory of policy gradients}\label{sec:recurrent_transeint}

In this section, we apply the recurrent-transient theory of Markov chains to undiscounted total-reward MDPs, introduce a new object that we term the \emph{transient visitation measure}, and establish a policy gradient theorem.

\subsection{Recurrent-transient classification of states}

\begin{definition}
Given a policy $\pi\in \Pi$, 
a state $s\in \cS$ is recurrent if its return time starting from $s$ is finite with probability $1$. Otherwise, $s$ is transient.
\end{definition}
Equivalently, if $n_s$ is the random variable representing the number of visits to state $s$ starting from $s$, then $s$ is recurrent if and only if $\expec_\pi[n_s]=\sum_{k=0}^{\infty} (P^\pi)^{k}(s, s)=\infty,
$ and otherwise it is transient \citep[Theorem 3.1.3]{bremaud2013markov}. 

Let $\pi \in \Pi$. For a given $P^\pi$, the states can be classified into recurrent and transient states, and the Markov chain can be canonically represented as follows \citep[Section 3.1.3]{bremaud2013markov}:
$$
P^\pi =
\begin{bmatrix}
\bar{R}^\pi & 0 \\
\bar{S}^\pi & \bar{T}^\pi
\end{bmatrix}, \qquad (P^\pi)^n =
\begin{bmatrix}
(\bar{R}^\pi)^n & 0 \\
\bar{S}_n^\pi & (\bar{T}^\pi)^n
\end{bmatrix},
$$
where $\bar{R}^\pi$, $\bar{T}^\pi$, and $\bar{S}^\pi$ represent transition probabilities among the recurrent states, among the transient states, and from transient to recurrent states, respectively. This canonical representation exists for any $\pi \in \Pi$, but the recurrent-transient classification of states and the corresponding canonical decomposition of $P^\pi$ may vary.


However, as the following proposition shows, the recurrent-transient classification remains invariant for all $\pi \in \Pi_+$. Recall that $\Pi_+ \subset \Pi$ denotes the set of policies that assign strictly positive probability to every action.

\begin{prop}\label{prop:classification}
The recurrent-transient classification of the states does not depend on the choice of $\pi\in \Pi_+$.
\end{prop}

We provide further clarification. For any $\pi\in \Pi_+$, the recurrent-transient classification is determined by the transition kernel $P$, not on the particular choice of $\pi\in \Pi_+$. While a policy $\pi\in \Pi\setminus\Pi_+$ (a policy that assigns zero probability to some actions) may induce a different classification, the algorithms we consider (as well as deep RL algorithms employing a softmax output layer) should be viewed as searching over $\Pi_+$ rather than the full set $\Pi$.

The canonical recurrent-transient decomposition provides the foundation for our analysis of the undiscounted expected total-reward setting. One key consequence of this classification is that the reward at any recurrent state must be zero if the value functions are finite.

\begin{lemma}\label{lem:recur_reward}
Under Assumption~\ref{assump_total} (finiteness of value function), for any $\pi \in \Pi$, if $s$ is a recurrent state, then $r^\pi(s) = 0$.
\end{lemma}
Now, define the \emph{transient matrix}:
\[
T^\pi=\begin{bmatrix}
    0&0\\0&\bar{T}^\pi
\end{bmatrix} \] 
i.e.,
\[
T^\pi(s_1,s_2)=
\left\{\begin{array}{ll}
P^\pi(s_1,s_2)&\text{if $s_1,s_2$ are both transient}\\
0&\text{otherwise.}
\end{array}\right.
\]
The transient matrix $T^\pi$ is known to have the following spectral property:
\begin{fact}\label{fact:spec}\citep[Lemma 8.3.20]{berman1994nonnegative}\label{fact:trans_spec}
Spectral radius of $T^\pi$ is strictly less than $1$.
\end{fact}

This is an important consequence of the recurrent-transient decomposition because the full probability matrix $P^\pi$ will necessarily have a unit spectral radius, and we will use this condition to argue certain convergence results.

By Lemma~\ref{lem:recur_reward}, we have $(\cP^\pi)^i r^\pi = (T^\pi)^i r^\pi$ for $i\in \mathbb{N}$, which implies 
\[
V^\pi =\sum_{i=0}^\infty (P^\pi)^i r^\pi= \sum_{i=0}^\infty (T^\pi)^i r^\pi.
\]
By Fact~\ref{fact:trans_spec} and the classical Neumann series argument, we have $(I-T^\pi)^{-1} = \sum_{i=0}^\infty (T^\pi)^i$.
These lead to the following reformulation of the value function.
\begin{lemma}\label{lem:value_func}
Under Assumption~\ref{assump_total}, $V^\pi  = (I-T^\pi)^{-1} r^\pi$ for any $\pi \in \Pi$.
\end{lemma}

\subsection{Continuity of $V^\pi$ on $\Pi_+$}
Returning to the pathological MDP of Figure~\ref{fig:example}, we observe that (1) at the discontinuous policy, the recurrent-transient classification of states changes, and (2) this transition occurs only on $\Pi\setminus\Pi_+$. Based on these observations and Proposition \ref{prop:classification}, we obtain the following continuity property of $V^\pi$.

\begin{lemma}\label{lem:continuity}
Under Assumption~\ref{assump_total}, the mappings $\pi \mapsto V^\pi$ and $\pi \mapsto V^\pi_\mu$ are continuous on $\Pi_+$ for a given $\mu$.  
\end{lemma}
(Recall $V^\pi_\mu = \mathbb{E}_{s\sim\mu} [V^\pi(s)]$.) In other words, the discontinuity described in Pathology~\ref{discontinuity} can arise only on the boundary of $\Pi$. Consequently, in our policy gradient methods, we restrict the search to policies $\pi \in \Pi_+$.


Next, define
\[
V_{+,\mu}^\star=
\sup_{\pi \in \Pi_+} \mathbb{E}_{s\sim\mu} [V^\pi(s)],\qquad
 V_{\mu}^\star=
 \max_{\pi \in \Pi}\mathbb{E}_{s\sim\mu} [V^\pi(s)],
\]
where $\mu$ is an initial state distribution.
By definition, $V_{\mu}^\star \ge V_{+,\mu}^\star$. Since our policy gradient methods search over $\Pi_+$, they should be thought of as optimizing for $V_{+,\mu}^\star$. However, the distinction between $V_{+,\mu}^\star$ and $V_{\mu}^\star$ disappears when rewards are nonnegative.  


\begin{lemma}\label{lem:reward}
 Under Assumption~\ref{assump_total}, if rewards are nonnegative, then the map  $\pi\mapsto V^\pi_\mu$ are continuous at optimal policies, and $V^\star_{+,\mu}= V^\star_\mu$ for a given $\mu$. 
\end{lemma}

Under this continuity, we can further show that policy gradient algorithms converge to $V_{\mu}^\star$.

\subsection{Transient visitation measure}

Conventionally, the state visitation measure is defined as
$d_{s_0}^\pi (s) = (1-\gamma)\sum_{i=0}^{\infty}\mathrm{Prob}(s_i = s \mid s_0;\, \cP^\pi)$ in the discounted infinite-horizon setting with $\gamma<1$. In the undiscounted setting with $\gamma=1$, this object becomes undefined.

Instead, we define the \emph{transient visitation measure} in the undiscounted total-reward setting using the transient matrix $T^\pi$ as follows:
\[
\delta^\pi_{s_0}(s)
= \sum_{i=0}^{\infty}\mathrm{Prob}(s_i = s \mid s_0;\, T^\pi)
= e_{s_0}^{\intercal}(I-T^\pi)^{-1} e_{s},
\]
where $e_s$ is the $s$-coordinate vector. Also define $\delta^\pi_{\mu}= \expec_{s_0 \sim \mu} [\delta^\pi_{s_0}]$ for an initial state distribution $\mu$. Note that this transient visitation measure is not a probability measure, and, importantly, $\max_{s, s_0 \in \cS} \delta^\pi_{s_0}(s) < \infty$ by Fact \ref{fact:trans_spec}. The transient visitation measure only considers transitions between transient states since these are sufficient to compute the value function, as shown in Lemma \ref{lem:value_func}.

With the transient visitation measure, we can obtain a performance difference lemma in the undiscounted total-reward setup, which will be crucially used in the analysis of policy gradient algorithms.
\begin{lemma}[Transient performance difference lemma]\label{lem:pdl} Under Assumption~\ref{assump_total},
for $\pi, \pi' \in \Pi$ and a given $\mu$, if $V^{\pi'}(s)=0$ for all recurrent states $s$ of $\cP^\pi$, then 
\begin{align*}
V_\mu^{\pi}\! -\! V_\mu^{\pi'}
\!\!&= \! \sum_{s' \in \cS}
\sum_{a \in \cA}
\!Q^{\pi'}(s',a)\!\big(\pi(a \mid s')-\pi'(a \mid s')\big) 
 \delta^{\pi}_\mu(s').
\end{align*}
\end{lemma}
 

\subsection{Transient policy gradient}
We are now ready to present the policy gradient theorem in the undiscounted total-reward setup. Consider the optimization problem
\[\max_{\theta \in \Theta }  V^{\pi_{\theta}}_{\mu},
\]
where $\{\pi_{\theta} \mid \theta \in \Theta\subset \real^d\}$ is a set of differentiable parametric policies with respect to $\theta$. Based on previous machinery, we establish the following policy gradient theorem.   
\begin{theorem}[Transient policy gradient]\label{thm::policy_gd}
Under Assumption~\ref{assump_total}, for $\pi_{\theta} \in \Pi_{+}$, 
\begin{align*}
    \nabla_{\theta} V^{\pi_{\theta}}_{\mu}
&= \sum_{s\in \cS}\delta_{\mu}^{\pi_\theta}(s)
\sum_{a\in \cA}
\nabla_{\theta} \pi_{\theta} (a \mid s)
Q^{\pi_{\theta}}(s,a)
\\&= \sum_{s\in \cS}\delta_{\mu}^{\pi_\theta}(s)
\mathop{\mathbb{E}}_{a\sim \pi_\theta(\cdot \,|\, s)}[
\nabla_{\theta} \log\pi_{\theta} (a \mid s)
Q^{\pi_{\theta}}(s,a)
].
\end{align*}

\end{theorem}
In the following sections, we use this transient policy gradient to analyze the projected policy gradient and natural policy gradient algorithms.

\section{Convergence of projected policy gradient}\label{sec:projecetd_gd}
In this section, we study the convergence of the projected policy gradient algorithm with direct parameterization: 
\[
\pi_{\theta} (a \,|\, s) = \theta_{s,a},
\]
where $\theta \in \real^{|\cS|\times |\cA|}$ satisfying $\sum_{a \in \cA} \theta_{s,a}=1$ and $\theta_{s,a}\ge 0$ for all $s\in \cS,a\in \cA$. With this direct parameterization, we do not distinguish between the policy $\pi_\theta$ and the parameter $\theta$, and we use $\pi_k$ to denote the iterates of the algorithm.

When using a direct parametrization, we require a mechanism to ensure $\theta$ remains nonnegative and normalized throughout the algorithm. So, we consider the \emph{projected} policy gradient: 
\begin{align*}
    \pi_{k+1} &= \textbf{proj}_{C}\left(\pi_{k}+\eta_k \nabla V_\mu^{\pi_k}  \right)\qquad\text{for }k=0,1,\dots, 
\end{align*}
where $C$ is a nonempty closed convex subset of $\Pi$. Usually, $C=\Pi$. But in the undiscounted total-reward setup, we must avoid the (relative) boundary of $\Pi$ as the value function may be discontinuous there, so we consider the following $\alpha$-shrunken $\Pi$ so that the policy set:
\[
\Pi_{\alpha} = \{ \pi \,|\, \pi (a \,|\, s) \ge \alpha\text{ for all }s\in \cS,\,a\in \cA\}
\]
with  $\alpha \in (0,1/|\cA|)$.
For evaluating $\nabla V_\mu^{\pi_k}$, 
Theorem~\ref{thm::policy_gd} applied to the direct parametrization setup yields $\nabla V_\mu^{\pi_k} (s,a)= \delta_{\mu}^{\pi_\theta}(s) Q^{\pi_{\theta}}(s,a)$ for $\pi \in \Pi_{\alpha}$.

Normally, the convergence analysis of projected policy gradient requires smoothness (Lipschitz continuity of the gradient) of the value function. 
For that, we define 
\[
\max_{\pi \in \Pi_{\alpha}}{\infn{(I-T^\pi)^{-1}}} = C_\alpha <\infty.
\]
Note that the mapping $\pi \mapsto P^\pi$ is continuous, and thus mapping $\pi \mapsto T^\pi$ is continuous on $\Pi_+$ since $T^\pi$ can be viewed as the projection of $P^\pi$ onto the transient class, which it is fixed by Proposition \ref{prop:classification}. Therefore, $C_\alpha$ is finite since  $\Pi_{\alpha}$ is compact and $\infn{(I-T^\pi)^{-1}}$ is continuous with respect to $\pi$. 

\begin{lemma}[Smoothness of value function]\label{lem:smooth}
Under Assumption~\ref{assump_total}, for $\pi, \pi' \in \Pi_{\alpha} $, 
    \[\|\nabla V^\pi_\mu-\nabla V^{\pi'}_\mu\|_2 \le  2 R C^2_\alpha(C_\alpha+1)|\cA| \|\pi- \pi'\|_2. \]
\end{lemma}

Define $V^{\pi^\star_\alpha}=\max_{\pi \in \Pi_{\alpha} } V^\pi$ and $V^{\pi^\star_\alpha}_\mu = \mathbb{E}_{s\sim\mu} [V^{\pi^\star_\alpha}(s)]$.
Using the Lemma \ref{lem:smooth}, we obtain the following convergence result of the projected policy gradient algorithm.
\begin{theorem}\label{thm:proj_gd}
Under Assumption~\ref{assump_total}, for $\alpha \in (0,1)$, $\pi_0 \in \Pi_\alpha$, and given $\mu$ with full support, the projected policy gradient algorithm with step size $\eta=\dfrac{1}{2 R C^2_\alpha(C_\alpha+1)|\cA|}$ generates a sequence of policies $\{\pi_{k}\}^\infty_{k=1}$ satisfying
\[
V_\mu^{\pi^\star_\alpha}-V_\mu^{\pi_{k}}
\;\le\;
\frac{256 R|\mathcal S|\,|\mathcal A|C_\alpha^2(C_\alpha+1)}{k }\,
\left\|\frac{\delta_\mu^{\pi_{\alpha}^\star}}{\mu}\right\|_\infty^{2}.
\]
\end{theorem}
We defer the proofs to Appendix~\ref{appen:miss_5}, but we quickly note that the proof strategy closely follows \cite{xiao2022convergence}, which considers the discounted reward setup and uses the (non-transient) visitation measure.

Theorem~\ref{thm:proj_gd} shows that $V_\mu^{\pi_k}\rightarrow V_\mu^{\pi^\star_\alpha}$ with a sublinear rate, and since $V^{\pi^\star_{\alpha}}_\mu \rightarrow V^\star_{+,\mu}$  as $\alpha\rightarrow 0$,  if we choose a sufficiently small $\alpha$ such that $V_{\mu,+}^\star-V_{\mu}^{\pi_\alpha}< \frac{\epsilon}{2}$, we can obtain policy $\pi$ satisfying $V_{\mu,+}^\star-V_\mu^\pi <\epsilon$.  Moreover, we can define an iteration complexity for finding an $\epsilon$-optimal policy by projecting onto $\Pi_{\alpha}$ with the value of $\alpha$ chosen to be a function of $\epsilon$ with nonnegative reward.




\begin{corollary}\label{cor:proj_grad}
Assume the rewards are nonnegative. For any given $\epsilon \in (0,1)$ and $\mu $ with full support, 
set $\alpha = \frac{\epsilon}{2|\cS||\cA| \infn{\delta^{\pi^\star}_{\mu}} \infn{Q^{\pi^\star}}} $ and
let the step size be  $\eta=\dfrac{1}{2 R C^2_{\alpha}(C_{\alpha}+1)|\cA|}$.
Then, under Assumption~\ref{assump_total}, for $\pi_0 \in \Pi_\alpha$, the iterates of projected policy gradient $\pi_k$ are $\epsilon$-optimal policy for 
\[ 
k\ge 512(1/\epsilon)  R|\mathcal S|\,|\mathcal A|C_{\alpha}^2(C_{\alpha}+1)
\left\|\frac{\delta_\mu^{\pi_{\alpha}^\star}}{\mu}\right\|_\infty^{2}.
\]
\end{corollary}




In Theorem \ref{thm:proj_gd}, we establish convergence to $\epsilon$-optimality with 
$\epsilon = V_{\mu}^{\star} - V_\mu^{\pi^\star_\alpha}$, which can be made arbitrarily small by taking $\alpha>0$ to be small. However, we do not have convergence to exact optimality as $k \to \infty$, where $k$ is the iteration count. Moreover, the convergence rate is sublinear, and the constant factor depends on the sizes of the state and action spaces, which may be quite large. These shortcomings are addressed by the analysis of the natural policy gradient method presented in the next section.



\section{Convergence of natural policy gradient}\label{sec:natural_gd}
In this section, we study convergence of natural policy gradient with softmax parametrization:
\[\pi_{\theta} (a \,|\, s) = \frac{\exp(\theta_{s,a})}{
\sum_{a'}\exp(\theta_{s,a'}).
}\]
where $\theta \in \real^{|\cS|\times|\cA|}$.
The softmax parametrized policy automatically satisfies $\pi_\theta\in \Pi_+$, so a projection mechanism is no longer necessary.

For a given $\mu$ with full support, the natural policy gradient algorithm, which can be seen as a policy gradient with the Fisher information matrix, is
\begin{align*}    
\begin{aligned}    
\theta^{k+1}
&= \theta^{k} + \eta_k F_{\mu}(\theta^{k}\big)^{\dagger}\,
  \nabla_{\theta} V^{\pi_k}_\mu \end{aligned}
\quad\text{for }k=0,1,\dots,
\end{align*}
where $F_{\mu}(\theta^k)
= \sum_{s \in \cS}\delta^{\pi_\theta}_{\mu}(s)
  \mathbb{E}_{a \sim \pi_{\theta}(\cdot \mid s)}
  \Big[ \nabla_{\theta}\log \pi_{\theta}(a \mid s)\,
        \big(\nabla_{\theta}\log \pi_{\theta}(a \mid s)\big)^{\top}
  \Big]$ and $\dagger$ denotes the Moore--Penrose pseudoinverse. It is known that natural policy gradient algorithm can also be expressed as
\begin{align*}
    \pi_{k+1}( a \,|\, s) &= \pi_{k}( a \,|\, s) \frac{\exp(\eta_k Q^{\pi_k}(s,a))}{z_s^k} \\& \propto
    \pi_{0}( a \,|\, s) \exp\Big(
    \sum_{i=0}^k
    \eta_i Q^{\pi_i}(s,a)\Big)
\end{align*}
for all $    s\in\cS$, $a\in\cA$, and $k=0,1,\dots$,
where \[
z^k_s = \sum_{a \in \cA}\pi_{k}( a \,|\, s) \exp(\eta_k Q^{\pi_k}(s,a))\] 
is a normalization factor. In the online learning literature, this update rule is also known as multiplicative weights updates \citep{freund1997decision} and the update rule does not depend on the initial state distribution $\mu$, as the pseudoinverse of the Fisher information removes this dependency.

\subsection{Sublinear convergence with constant step size}

We establish the sublinear convergence of the policy gradient algorithm with a constant step size. As a first step in our analysis, we state the following lemma, which ensures that the policies generated by the natural policy gradient algorithm improve monotonically.

\begin{lemma}\label{lem:ngd_descent}
Under Assumption~\ref{assump_total},  for $\pi_0 \in \Pi_+$ and given $\mu$ with full support, the natural policy gradient with constant step size $\eta>0$ generates a sequence of policies $\{\pi_{k}\}^\infty_{k=1}$ satisfying
\[
V_\mu^{\pi_k} \le V_\mu^{\pi_{k+1}}.
\]
\end{lemma}
Next, define
\[
\mathrm{KL}_{\delta_\mu^{\pi}}\!\big(\pi,\,\pi'\big)
 = \sum_{s\in\cS} \delta_{\mu}^{\pi}(s)
      \mathrm{KL}\big(\pi (\cdot\,|\,s),\,\pi_{k}(\cdot \,|\, s)\big),
      \]      
     where $\mathrm{KL}(p,q)=\sum^n_{i=1}p_i \log (p_i/q_i)$ for $p,q \in \cM(\cA)$, and also define 
     $
\infn{(I-T^\pi)^{-1}} = C_\pi <\infty.$
Combining Lemmas~\ref{lem:pdl} and \ref{lem:ngd_descent}, we obtain the following sublinear convergence.
\begin{theorem}\label{thm:ngd}
Under Assumption~\ref{assump_total}, for any $ \pi, \pi_0\in \Pi_+$ and given $\mu$ with full support, the natural policy gradient with constant step size $\eta>0$ generates a sequence of policies $\{\pi_{k}\}^\infty_{k=1}$ satisfying
\[
V_\mu^{\pi}-V_\mu^{\pi_k} 
\le \frac{1}{k+1}\left(\frac{\mathrm{KL}_{\delta_\mu^{\pi}}(\pi,\pi_0)}{\eta}
+ 2C_\pi \infn{V^\star_{+}}\right).
\]
Hence, $V_\mu^{\pi_k} \rightarrow V^\star_{+,\mu}$ as $k\rightarrow \infty$.
\end{theorem}
Appendix~\ref{appen:miss_6} provides the proof, which closely follows \cite{xiao2022convergence}.

Unlike the projected policy gradient algorithm, the convergence rate of the policy gradient method is independent of the size of the state or action space. Moreover, if we assume the rewards are nonnegative, the convergence result can be strengthened from
$V_\mu^{\pi_k} \to V^\star_{+,\mu}$ to
$V_\mu^{\pi_k} \to V^\star_\mu$.


\begin{corollary}\label{cor:ngd}
Assume the rewards are nonnegative. Under Assumption~\ref{assump_total},  for any $\pi, \pi_0 \in \Pi_+$ and given $\mu$ with full support, the natural policy gradient algorithm with constant step size $\eta>0$ generates a sequence of policies $\{\pi_{k}\}^\infty_{k=1}$ satisfying
\[
V_\mu^{\star}-V_\mu^{\pi_k} 
\le \frac{1}{k+1}\left(\frac{\mathrm{KL}_{\delta^{\pi^\star}_{\mu}}(\pi^\star,\pi_0)}{\eta}
+ C_{\pi^\star}\infn{V^\star}\right)
\]
for any optimal policy $\pi^\star$.
\end{corollary}

\subsection{linear convergence with adaptive step size}

Next, we present the \emph{linear} convergence rate of the natural policy gradient algorithm with adaptive step size. For a given $\mu$ with full support, define  $
\vartheta^\pi_{\mu}
= \left\|
\frac{\delta_{\mu}^{\pi}}{\mu}
\right\|_{\infty}\in [1,\infty),$
which represent the distribution mismatch between $\mu $ and $\delta_{\mu}^{\pi}$.
\begin{theorem}\label{thm:ngd_linear}
Under Assumption~\ref{assump_total}, for  $\pi_0 \in \Pi_+$ and $\mu$ with full support, the natural policy gradient algorithm with step sizes $(\vartheta^\pi_{\mu}-1)\eta_{k+1}\ge \vartheta^\pi_{\mu}\eta_k>0$ generates a sequence of policies $\{\pi_{k}\}^\infty_{k=1}$ satisfying   
\[
V_\mu^\pi-V_\mu^{\pi_{k}} 
\le
\left(1-\frac{1}{\vartheta^\pi_{\mu}}\right)^{\!k}
\left(
V_\mu^{\pi}- V_\mu^{\pi_{0}}
+ \frac{\mathrm{KL}_{\delta_\mu^{\pi}}(\pi, \pi_{0})}{\eta_0(\vartheta^\pi_\mu-1)}
\right).
\]
 Hence, $V_\mu^{\pi_k} \rightarrow V^\star_{+,\mu}$ as $k\rightarrow \infty$.
\end{theorem}

As the distribution mismatch decreases, we can see that the natural policy gradient converges faster.
Again, the convergence rate is independent of the size of the state or action space, and we can strengthen the convergence result if we assume the rewards are nonnegative.

\begin{corollary}\label{cor:ngd_linear}
Assume the rewards are nonnegative. Under Assumption~\ref{assump_total},  for $\pi_0 \in \Pi_+$ and given $\mu$ with full support, the natural policy gradient algorithm with step size $(\vartheta^{\pi^\star}_{\mu}-1)\eta_{k+1}\ge \vartheta^{\pi^\star}_{\mu}\eta_k>0$ generates a sequence of policies $\{\pi_{k}\}^\infty_{k=1}$ satisfying   
\[
V_\mu^\star-V_\mu^{\pi_{k}} 
\le
\left(1-\frac{1}{\vartheta^{\pi^\star}_{\mu}}\right)^{\!k}
\left(
V_\mu^{\star}- V_\mu^{\pi_{0}}
+ \frac{\mathrm{KL}_{\delta_\mu^{\pi^\star}}(\pi^\star, \pi_{0})}{\eta_0(\vartheta^{\pi^\star}_\mu-1)}
\right)
\]    
for any optimal policy $\pi^\star$.
\end{corollary}
 Although the adaptive step size yields a linear convergence rate, it requires knowledge of $\vartheta_\mu^{\pi}$ to set the step sizes. In contrast, a constant step size always guarantees a sublinear rate.


\section{Stochastic natural policy gradient}\label{sec:SNPG}
In this section, we extend the analysis of the natural policy gradient to the sampling setting in which the transition probabilities are unknown. Specifically, we assume access to a generative model \cite{kearns1998finite}, which provides independent samples of the next state for any given state and action.

\subsection{Approximating $Q$-function with a generative model}
With generative model, for a given policy $\pi$ and any state-action pair $(s,a)\in\mathcal{S}\times\mathcal{A}$, we can
generate independent trajectories of horizon $H$, i.e.,
\[
\left\{
\bigl(s_0,a_0\bigr),\bigl(s_1,a_1\bigr),\ldots,\bigl(s_{H-1},a_{H-1}\bigr)
\ \middle|\ 
s_0=s,\ a_0=a
\right\}.
\]
 With this samples, we consider following $Q$-estimator to approximate state-action value function.
\[\Tilde{Q}^\pi=\frac{1}{N}\sum^{N}_{j=1} \Tilde{Q}_j^\pi \quad \text{where} \quad \Tilde{Q}_j^\pi(s_0,a_0)=\sum^H_{i=0}r_j(s_i,a_i)\]
for all $s_0 \in\cS$, $a_0\in\cA$. 

Next, we  define \emph{transient half-life} for given $\pi$ as
\[t_{\frac{1}{2},\pi} = \argmin_{t \ge 1} \infn{(T^\pi)^t}\le \frac{1}{2},\]
 where $\infn{\cdot}$ denotes the maximum absolute row sum of the matrix. This transient half-life is always finite due to Fact \ref{fact:spec}.

We now present sample complexity on the $Q$-estimator with generative model.
\begin{theorem}\label{thm:PE}
    Let $\epsilon>0$ and $ \delta>0$. Under Assumption~\ref{assump_total},  for given $\pi \in \Pi$, with $1-\delta$ probability, $\infn{\tilde{Q}^\pi - Q^\pi} \le \epsilon $ with sample complexity 
    \[\widetilde{\mathcal{O}}\left( \frac{t_{\frac{1}{2},\pi} ^3 R^2 |\cS||\cA|}{\epsilon^2}\right)\]
    where $\widetilde{\mathcal{O}}$ ignores all logarithmic factors.
\end{theorem}
Note that this holds for arbitrary $\pi \in \Pi$, and if we  priori know classification of states, $|\cS|$ in sample complexity of Theorem \ref{thm:PE} could be reduced to number of transient states.

\subsection{Sample complexity of stochastic natural policy gradient}

By approximating $Q$-function, stochastic natural policy gradient can be expressed as
\begin{align*}
&\Tilde{Q}^{\pi_k}(s_0,a_0)=\frac{1}{N}\sum^{N}_{j=1}\sum^H_{i=0}r_j(a_i,s_i)\\
    &\pi_{k+1}( a \,|\, s) = \pi_{k}( a \,|\, s) \frac{\exp(\eta_k \tilde{Q}^{\pi_k}(s,a))}{z_s^k}
\end{align*}
for all $s_0,    s\in\cS$, $a, a_0\in\cA$, and $k=0,1,\dots$,
where $ z^k_s = \sum_{a \in \cA}\pi_{k}( a \,|\, s) \exp(\eta_k Q^{\pi_k}(s,a))$. Note that unlike projected policy gradient, natural policy gradient only requires information of state-action value function.

\begin{figure*}[ht]
    \centering
\includegraphics[width=0.9\linewidth]{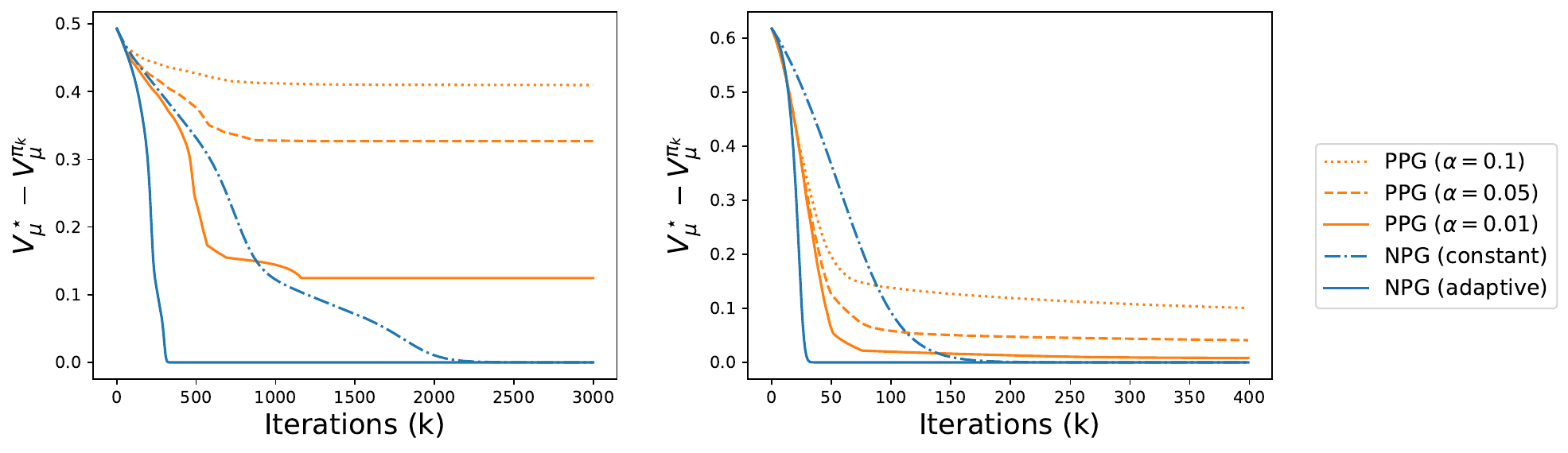}
    \caption{Comparison of projected policy gradient (PPG) and natural policy gradient (NPG) algorithms in (left) Frozenlake and (right) Cliffwalk.
    The limit of the projected policy gradient algorithm gets closer to the optimum as $\alpha>0$ gets smaller.}
    \label{fig:pg}
\end{figure*}

Now we present sample complexity of stochastic natural policy gradient. We defer the proofs to Appendix \ref{appen:D}, but we note
that the proof strategy closely follows  \cite{xiao2022convergence}.

\begin{theorem}\label{thm:SNPG}
    Let $\epsilon>0$ and $\delta>0$. Under Assumption~\ref{assump_total}, for any $\pi, \pi_0 \in \Pi_+$ and given $\mu$ with full support, with probability $1-\delta$, the iterate of stochastic natural policy gradient with constant step size $\eta>0$ and $K=2\Big(\frac{\text{KL}_{\rho_\mu^{\pi}}(\pi,\pi_0)}{\eta}
+ 2C_{\pi}\infn{V^\star_{+}}\Big)/\epsilon$ satisfies $V_\mu^{\pi}-V_{\mu}^{\pi_K}\le \epsilon$ with sample complexity 
\[\widetilde{\mathcal{O}}\left(\frac{ t_{\frac{1}{2}}^3 (\vartheta^\pi_\mu)^2 R^2 C_\pi^5\left( \frac{\text{KL}_{\delta_\mu^{\pi}}(\pi, \pi_0)}{\eta}\right)^3\infn{V^\star_{+}}^3 |\cS||\cA|}{\epsilon^5}  \right) \]
and with adaptive step size $(\vartheta^\pi_{\mu}-1)\eta_{k+1}\ge \vartheta^\pi_{\mu}\eta_k>0$ and $K = \log \left(2\left(
V_\mu^\pi-V_\mu^{\pi_0}
+
\frac{\text{KL}_{\delta^\pi_\mu}(\pi, \pi_{0})}{\eta_0(\vartheta^\pi_\mu - 1)} \right)/\epsilon
\right) $ satisfies $V_\mu^{\pi^\star}-V_\mu^{\pi_K}\le \epsilon$ with sample complexity 
\[\widetilde{\mathcal{O}}\left(\frac{t_{\frac{1}{2}}^3 (\vartheta^\pi_\mu)^2C_K^2 R^2 |\cS||\cA|}{\epsilon^2} \right)\] 
where $\widetilde{\mathcal{O}}$ ignores all logarithmic factors, $C_{K}= \max_{0\le k \le K-1}\{C_{\pi_{i}},C_\pi\}$,  and $t_{\frac{1}{2}}= \max_{0\le k \le K-1}t_{\frac{1}{2},\pi_k}$.
\end{theorem}

Note that sample complexity of adaptive step size $\widetilde{\mathcal{O}}(1/\epsilon^2)$ is more efficient that of constant step size $\widetilde{\mathcal{O}}(1/\epsilon^5)$.
We can also derive similar rates when the assume the rewards are nonnegative as follows. 

\begin{corollary}\label{cor:sngd}
Assume the rewards are nonnegative.    Let $\epsilon>0$ and $\delta>0$. Under Assumption~\ref{assump_total}, for $\pi_0 \in \Pi_+$ and given $\mu$ with full support, with probability $1-\delta$, the policy error of stochastic natural policy gradient with constant step size $\eta>0$ and $K=2\Big(\frac{\text{KL}_{\rho_\mu^{\pi^\star}}(\pi^\star,\pi_0)}{\eta}
+ C_{\pi^\star}\infn{V^\star}\Big)/\epsilon$ satisfies $V_\mu^{\star}-V_\mu^{\pi_K}\le \epsilon$ with sample complexity 
\[\widetilde{\mathcal{O}}\left(\frac{ t_{\frac{1}{2}}^3 (\vartheta^{\pi^\star}_\mu)^2 R^2 C_{\pi^\star}^5\left( \frac{\text{KL}_{\delta_\mu^{\pi^\star}}(\pi^\star, \pi_0)}{\eta}\right)^3\infn{V^\star}^3 |\cS||\cA|}{\epsilon^5}  \right) \]
and  with adaptive step size $(\vartheta^{\pi^\star}_{\mu}-1)\eta_{k+1}\ge \vartheta^{\pi^\star}_{\mu}\eta_k>0$ and $K = \log \left(2\left(
V_\mu^{\star}-V_\mu^{\pi_0}
+
\frac{\text{KL}_{\delta^{\pi^\star}_\mu}({\pi^\star}, \pi_{0})}{\eta_0(\vartheta^\pi_\mu - 1)} \right)/\epsilon
\right) $ satisfies $V_\mu^{\pi^\star}-V_\mu^{\pi_K}\le \epsilon$ with sample complexity 
\[\widetilde{\mathcal{O}}\left(\frac{t_{\frac{1}{2}}^3 (\vartheta^{\pi^\star}_\mu)^2C_K^2 R^2 |\cS||\cA|}{\epsilon^2} \right)\] 
for any optimal policy $\pi^\star$ where $\widetilde{\mathcal{O}}$ ignores all logarithmic factors, $C_{K}= \max_{0\le k \le K-1}\{C_{\pi_{i}},C_{\pi^\star}\}$,  and $t_{\frac{1}{2}}= \max_{0\le k \le K-1}t_{\frac{1}{2},\pi_k}$.
\end{corollary}



\section{Experiments}
For the experiments, we consider two toy examples: Frozenlake with $4\times4$ states and 4 actions and CliffWalk with $3\times7$ states and 4 actions. We use nonnegative rewards which ensures $V_\mu^\star = V_{+,\mu}^\star$  by Lemma~\ref{lem:reward}, and uniform initial state distribution. Further details are provided in Appendix~\ref{appen:experiments}. 


We run the projected policy gradient method with $\alpha \in \{0.1, 0.05, 0.01\}$ and the natural policy gradient method with both constant and adaptive step sizes. All algorithms are implemented using the transient policy gradient with transient visitation measure. For Frozenlake, we use $\{0.1\cdot 1.01^k\}^\infty_{k=0}$  for the adaptive step size of natural policy gradient, where $k$ is the number of iterations, and $0.1$ for others. For CliffWalk, we use $\{0.05\cdot 1.1^k\}^\infty_{k=0}$ for the adaptive step size and $0.05$ for others.

The results are shown in Figure~\ref{fig:pg}. The natural policy gradient with adaptive step size exhibits the fastest convergence rate among the algorithms, as the guaranteed linear rate of Corollary~\ref{cor:ngd_linear} predicts. Note that both natural policy gradients converge to $V^\star_\mu$ while the projected policy gradient converges to $V_\mu^{\pi^\star_\alpha}$ for each $\alpha$, and smaller $\alpha$ makes projected policy gradient converge closer to $V_\mu^{\star}$ since $V_\mu^{\pi^\star_\alpha}$ increases monotonically to $V_\mu^\star$ as $\alpha \to 0$.

Additionally, we run an experiment with pathological MDP of Figure \ref{fig:example}, shown in Appendix~\ref{appen:experiments}.



\section{Conclusion}
In this work, we present the first analysis of policy gradient methods for undiscounted expected total-reward infinite-horizon MDPs. Our approach combines the classical recurrent-transient theory from Markov chain theory with prior analysis techniques for policy gradient methods. Specifically, we first establish invariance of the classification of MDP states on $\Pi_+$, where the value function is continuous, and define a new transient visitation measure that leads to a transient policy gradient. Based on this machinery, we establish non-asymptotic convergence rates for projected policy gradient and natural policy gradient in the undiscounted total-reward setting. Lastly, we present the sample complexity of stochastic natural policy gradient by approximating the $Q$-function with a generative model.

Our framework opens the door to several directions for future work. One direction is to extend our results to function approximation in a sampling setting, where restricted parametric policies may not include the optimal policy, and the estimation and optimization errors from finite samples must be quantified. Another promising direction is to establish the convergence of the naive policy gradient with softmax parameterization, without preconditioning by the Fisher information matrix.

Finally, we highlight that recurrent-transient classification of MDP states is a fundamental and broadly applicable technique. Previously, the recurrent-transient theory was also applied to improve the convergence of policy iteration independently of the discount factor \citep{fox1968policy,bertsekas1991analysis,scherrer2013improved}. We expect that this technique can be used to analyze a wide range of RL algorithms in the undiscounted total-reward setting.
\bibliography{icml2026}
\bibliographystyle{icml2026}

\newpage
\appendix
\onecolumn

\section{Omitted proofs in Section \ref{sec:recurrent_transeint}}

For definitions of basic concepts of transient-recurrent theory such as irreducible class,
communicating class, closedness, etc., please refer to \citet[Chapters 2 and 3]{bremaud2013markov}

\subsection{Proof of Proposition \ref{prop:classification}}
\begin{proof} For any $\pi, \pi' \in \Pi_+$, $P^\pi (s,s') \neq 0 $ if and only if $P^{\pi'} (s,s') \neq 0$ for $s,s' \in \cS$ by definition of $\Pi_+$. This implies $s$ and $s'$ of $P^\pi$ communicate if and only if $s$ and $s'$ of $P^{\pi'}$ communicate, and thus communicating class is invariant for $\pi \in \Pi_+$. It is known that states in communicating class are all transient or recurrent \citep[Theorem 3.1.6]{bremaud2013markov}. Next, the set of states is closed in $P^{\pi}$ if and only if it is closed in $P^{\pi'}$.  Therefore, since  communicating class is closed if and only if it is recurrent in finite states MDP \citep[Theorem 3.2.8]{bremaud2013markov}, we obtain the desired result.  
\end{proof}

\subsection{Proof of Lemma \ref{lem:recur_reward}}

\begin{proof}
 Since  $\sum_{k=0}^{\infty} (P^\pi)^{k}(s, s)=\infty$ for recurrent state $s$, $r^\pi(s)=0$ to satisfy Assumption \ref{assump_total}. 
\end{proof}

By Lemma \ref{lem:recur_reward} and \ref{lem:value_func}, we directly obtain the following Corollary.
\begin{corollary}\label{cor:val_zero}
  Under Assumption~\ref{assump_total},  $V^\pi(s)=0$ for all recurrent states $s$.
\end{corollary}

\subsection{Proof of Lemma \ref{lem:continuity}}
\begin{proof}
$\pi \mapsto P^\pi$ is continuous, and by Proposition \ref{prop:classification}, $\pi \mapsto T^\pi$ is also continuous on $\Pi_+$. Since $(I-T^\pi)^{-1}$ is continuous with respect to $T^\pi$, by Lemma \ref{lem:value_func}, $V^\pi$ and $V^\pi_\mu$ are continuous with respect to $\pi$.
\end{proof}

\subsection{Proof of Lemma \ref{lem:pdl}}



\begin{proof}
\begin{align*}
    V^\pi(s)-V^{\pi'}(s) &= \sum_{a\in \cA} Q^\pi(s,a) \pi(a \,|\, s) - \sum_{a\in \cA}  Q^{\pi'}(s,a) \pi'(a \,|\, s) 
    \\&= \sum_{a\in \cA}  Q^{\pi'}(s, a) (\pi(a \,|\, s) -\pi'(a \,|\, s))+\sum_{a\in \cA} ( Q^\pi(s, a)-Q^{\pi'}(s, a)) \pi(a\,|\, s)
    \\&= \sum_{a\in \cA}  Q^{\pi'}(s, a) (\pi(a \,|\, s) -\pi'(a \,|\, s))+\sum_{s'\in \cS}\sum_{a\in \cA} P(s' \,|\, s, a) (V^\pi(s')-V^{\pi'}(s'))\pi(a \,|\, s)
       \\&= \sum_{a\in \cA}  Q^{\pi'}(s, a) (\pi(a \,|\, s) -\pi'(a \,|\, s))+ (\cP^\pi (V^\pi-V^{\pi'}))(s)
\end{align*}
where we used Bellman equation in third equality. 
Let $u(s) = \sum_{a\in \cA}  Q^{\pi'}(s, a) (\pi(a \,|\, s) -\pi'(a \,|\, s)).$ Then, we have
\begin{align*}
    V^\pi-V^{\pi'} &= u + P^\pi(V^\pi-V^{\pi'} ) \\&= u +  T^\pi(V^\pi-V^{\pi'})
\end{align*}
which further implies 
\[V^\pi-V^{\pi'} = (I-T^\pi)^{-1}u\]
and 
\[V_\mu^\pi-V_\mu^{\pi'} = \mu^\intercal (I-T^\pi)^{-1}u.\]
\end{proof}

\begin{corollary}\label{cor:pdl} Under Assumption~\ref{assump_total}, if $\pi, \pi' \in \Pi_+ $, 
\begin{align*}
V_\mu^{\pi} - V_\mu^{\pi'}
&=  \sum_{s' \in \cS}
\sum_{a \in \cA}
Q^{\pi'}(s',a)(\pi(a \mid s')-\pi'(a \mid s')) 
 \delta^{\pi}_\mu(s').
\end{align*}
\end{corollary}
\begin{proof}
By Lemma \ref{lem:recur_reward} and \ref{lem:value_func}, we obtain $V^\pi(s)=0$ for any recurrent state $s$. Thus by Proposition \ref{prop:classification}, condition of Lemma \ref{lem:pdl} is satisfied.
\end{proof}
\begin{corollary} \label{cor:pdl_reward} Under Assumption~\ref{assump_total}, if rewards are nonnegative, for $\pi \in \Pi$, 
\begin{align*}
V_\mu^{\star} - V_\mu^{\pi}
&=  \sum_{s' \in \cS}
\sum_{a \in \cA}
Q^{\pi}(s',a)(\pi^\star(a \mid s')-\pi(a \mid s')) 
 \delta^{\pi^\star}_\mu(s').
\end{align*}
\end{corollary}
\begin{proof}
      Since $V^\star \ge V^{\pi} \ge 0$ by definition of $V^\pi$, $\pi$ satisfies condition of Lemma \ref{lem:pdl}.
\end{proof}

\subsection{Proof of Lemma \ref{lem:reward}}

\begin{proof}
    By Corollary $5$, we have 
    \[V_\mu^{\star} - V_\mu^{\pi} \le  |\cS||\cA|\infn{Q^\pi}\infn{\pi^\star- \pi }\infn{\delta^{\pi^\star}_\mu}.\]
   Since $\infn{Q^\pi}$ is bounded by $\infn{Q^\star}$, $ 
\lim_{\pi \rightarrow \pi^\star }  V_\mu^{\pi}= V_\mu^{\star}$ and this implies $V^\star_{+,\mu}= V^\star_\mu$.
\end{proof}

\subsection{Proof of Theorem \ref{thm::policy_gd}}
For the proof of Theorem \ref{thm::policy_gd}, we first prove the following lemmas.
\begin{lemma}\label{lem:8}
    Under Assumption~\ref{assump_total}, for recurrent state $s$ of $\cP^\pi$ where $\pi \in \Pi_+$, $r(s,a)=0$. 
\end{lemma}
\begin{proof}
    If $r(s,a)\neq0$ for some $a \in \cA$, there exists $\pi \in \Pi_+$ such that $r^\pi(s)\neq 0$ since $\infn{r}\le R$. This is a contradiction by Lemma \ref{lem:recur_reward}. 
\end{proof}
Now we prove Theorem \ref{thm::policy_gd}.
\begin{proof}
Fix $\pi \in \Pi_+$. We first clarify  differentiability of $V^\pi$.
For $\triangle \pi \in \real^{|\cS|\times|\cA|}$ which represents change of policy, we define $P^{\pi+\triangle \pi}(s,s')=\sum_{a \in \cA} (\pi+\triangle \pi)(s,a)P(s' \,|\, s,a)$ and $r^{\pi+\triangle \pi}(s) = \sum_{a \in \cA} (\pi+\triangle \pi)(s,a)r(s,a)$ for all $s,s' \in \cS$. Then, if $s$  is recurrent, $r^{\pi+\triangle \pi}(s)=0$ by Lemma \ref{lem:8}.

   By definition of $\Pi_+$, there exist an open ball $B(\pi, \epsilon)$ such that $(\pi+\triangle \pi)(a \,|\, s)>0 $ for all $s\in\cS, a \in \cA$ and $\triangle \pi \in B(\pi, \epsilon)$. Then,   $P^\pi (s,s') \neq 0 $ if and only if $P^{\pi+\triangle \pi} (s,s') \neq 0$ for all $s,s'\in\cS$, and we can define perturbed transient matrix $T^{\pi+\triangle \pi}$ for $\pi \in \Pi_+$. 
   
   Since spectral radius is continuous with respect to entries of matrix \citep[Theorem 2.4.9.2]{horn2012matrix}, and $T^\pi$ is continuous with respect to $\pi \in \Pi_+$, there also exist open ball $B'(\pi, \epsilon) \subset B(\pi, \epsilon)$ such that spectral radius of $T^{\pi+\triangle \pi}$ is smaller than $1$. (Note that this argument is valid since set of transient class and recurrent class fixed by Proposition \ref{prop:classification}.) Thus 
   \[V^{\pi+\triangle \pi} =\sum^\infty_{i=0} (P^{\pi+\triangle \pi})^i r^{\pi+\triangle \pi} =\sum^\infty_{i=0} (T^{\pi+\triangle \pi})^i r^{\pi+\triangle \pi} = (I-T^{\pi+\triangle \pi})^{-1} r^{\pi+\triangle \pi}, \]
 and this implies well-definedness of value function on $\pi+\triangle \pi$ where $c \in \Pi_+$ and $\triangle \pi \in B'(\pi, \epsilon)$. Then, by Lemma \ref{lem:value_func}, differentiability of $T^\pi$ and $r^\pi$ on $\pi \in \Pi_+$ implies differentiability of $V^\pi$, and it can be easily seen that $T^\pi$ is differentiable with respect to $\pi$ since each element of $P^\pi$ is differentiable and transient class is fixed by Proposition \ref{prop:classification}. $r^\pi$ is obviously differentiable.

Therefore,
\begin{align*}
\nabla_{\theta} V^{\pi_{\theta}}_\mu &= \nabla_{\theta} \mu^\intercal (I-T^{\pi_\theta})^{-1}r^{\pi_\theta}    
\\& =\nabla_{\theta}(\mu^\intercal (I-T^{\pi_\theta})^{-1})r^{\pi_\theta}+\mu^\intercal (I-T^{\pi_\theta})^{-1}\nabla_\theta r^{\pi_\theta}
\\& =\mu^\intercal (I-T^{\pi_\theta})^{-1} \frac{\partial T^{\pi_\theta}}{\partial \theta}(I-T^{\pi_\theta})^{-1}r^{\pi_\theta}+ \mu^\intercal(I-T^{\pi_\theta})^{-1}\nabla_\theta r^{\pi_\theta}
\\& =\mu^\intercal (I-T^{\pi_\theta})^{-1} \frac{\partial \Theta P}{\partial \theta}V^{\pi_\theta}+ \mu^\intercal(I-T^{\pi_\theta})^{-1}\frac{\partial\Theta r}{\partial \theta} 
   \\&= \mu^\intercal (I-T^{\pi_{\theta}})^{-1} \frac{\partial \Theta}{\partial \theta}(P V^{\pi_{\theta}} + r)
   \\&=  \mu^\intercal(I-T^{\pi_{\theta}})^{-1} \frac{\partial \Theta}{\partial \theta}Q^{\pi_{\theta}}
\end{align*}
where third equality comes from matrix calculus $\frac{\partial A^{-1}}{\partial \theta} = A^{-1}\frac{\partial A}{\partial \theta}A^{-1}$ for $A(\theta) \in \real ^{|\cS| \times |\cS|}$ and forth equality is from fact that $\Theta \in  \real^{|\cS| \times |\cS||\cA|} $ is matrix form of policy $\pi_\theta$ satisfying $\Theta P = P^{\pi_\theta}$ and $\Theta r = r^{\pi_\theta}$ and $\frac{\partial T^{\pi_\theta}}{\partial \theta} V^{\pi_\theta}  = \frac{\partial  P^{\pi_\theta}}{\partial \theta}V^{\pi_\theta}$  by Corollary \ref{cor:val_zero} and Proposition \ref{prop:classification}.
\end{proof}

\section{Omitted proofs in Section \ref{sec:projecetd_gd}}\label{appen:miss_5}
\subsection{Proof of Lemma \ref{lem:smooth}}
\begin{proof}
We basically follow the proof strategy of \cite{agarwal2021theory, mei2020global} .    Let $ \theta_{\beta} = \theta + \beta u$. Then, with direct parametrization,
    \[ \max_{\|u\|_2=1}\sum_a \left[ \frac{\partial \pi_{\theta_\beta}(a|s)}{\partial \beta} \Big|_{\beta=0} \right]  \le  \sqrt{|\cA|}, \qquad  \sum_a \left[ \frac{\partial^2 \pi_{\theta_\beta}(a|s)}{\partial \beta^2} \Big|_{\beta=0} \right]  =0.\]

    Note that $T^{\pi_{\theta_\beta}} \in \mathbb{R}^{|\cS| \times |\cS|}$ as
\begin{align*}
[T^{\pi_{\theta_\beta}}]_{(s,s')} &= \sum_a \pi_{\theta_\beta}(a|s) \cdot P(s'|s,a) \qquad  \text{for all } s,s' \in \mathcal{T}     
\\& =0 \qquad \qquad \qquad \qquad \qquad \quad \,\,\text{otherwise}
\end{align*}
where $\mathcal{T}$ is invariant transient class. Then, the derivative with respect to $\beta$ is
\[
\left[ \frac{\partial T^{\pi_{\theta_\beta}} }{\partial \beta} \Big|_{\beta=0} \right]_{(s,s')}
= \sum_a \left[ \frac{\partial \pi_{\theta_\beta}(a|s)}{\partial \beta} \Big|_{\beta=0} \right] 
\cdot P(s'|s,a)
\]
for $ s,s' \in \mathcal{T}$,   and for any vector $x \in \mathbb{R}^{|\cS|}$, we have 
\begin{align*}
\left\| \frac{\partial T^{\pi_{\theta_\beta}} }{\partial \beta} \Big|_{\beta=0} x \right\|_\infty
&= \max_{s\in \mathcal{T}} \left| \sum_{s'\in \mathcal{T}} \sum_a \left[ \frac{\partial \pi_{\theta_\beta}(a|s)}{\partial \beta} \Big|_{\beta=0} \right] 
\cdot P(s'|s,a) \cdot x(s') \right| \\
&\le \max_{s\in \mathcal{T}} \sum_a \sum_{s'\in \mathcal{T}} P(s'|s,a) 
\cdot \left| \frac{\partial \pi_{\theta_\beta}(a|s)}{\partial \beta} \Big|_{\beta=0} \right|
\cdot \|x\|_\infty \\
&\le \max_{s\in \mathcal{T}} \sum_a \left| \frac{\partial \pi_{\theta_\beta}(a|s)}{\partial \beta} \Big|_{\beta=0} \right| \cdot \|x\|_\infty . 
\end{align*}
Therefore, 
\[ \max_{\|u\|_2=1}\left\| \frac{\partial T^{\pi_{\theta_\beta}} }{\partial \beta} \Big|_{\beta=0} x \right\|_\infty \le  \sqrt{\cA} \cdot \|x\|_\infty\]
Similarly, taking second derivative with respect to $\beta$,
\[
\left[ \frac{\partial^2 T^{\pi_{\theta_\beta}} }{\partial \beta^2} \Big|_{\beta=0} \right]_{(s,s')}
= \sum_a \left[ \frac{\partial^2 \pi_{\theta_\beta}(a|s)}{\partial \beta^2} \Big|_{\beta=0} \right]
\cdot P(s'|s,a)=0.
\]

Next, consider the state value function of $\pi_{\theta_\beta}$,
\[
V^{\pi_{\theta_\beta}}(s) = e_s^\top M^{\pi_{\theta_\beta}} r^{\pi_\beta},
\]
where
\[
M^{\pi_{\theta_\beta}} = (I -T^{\pi_{\theta_\beta}})^{-1}, \qquad r^{\pi_{\theta_\beta}}(s) = \sum_a \pi_{\theta_\beta}(a|s) \cdot r(s,a) \quad \text{for all} \,s \in \cS.
\]
Since $[T^{\pi_{\theta_\beta}} ]_{(s,s')} \ge 0$,  for all $s,s'$, we have $[M^{\pi_{\theta_\beta}}]_{(s,s')} \ge 0$.  
Suppose $\infn{M^{\pi_{\theta_\beta}}} \le C_\beta$. Then, for any vector $x \in \mathbb{R}^S$,
\begin{align*}
\|M^{\pi_{\theta_\beta}}x\|_\infty 
\le C_\beta \cdot \|x\|_\infty.
\end{align*}
Note that $\infn{r^{\pi_{\theta_\beta}}} \le R$. Thus, we have



\begin{align*}
\left\| \frac{\partial r^{\pi_{\theta_\beta}}}{\partial \beta} \right\|_\infty
&= \max_s \left| \frac{\partial r^{\pi_{\theta_\beta}}(s)}{\partial \beta} \right| \\& \le R \max_s \sum_a \left[ \frac{\partial \pi_{\theta_\beta}(a|s)}{\partial \beta} \Big|_{\beta=0} \right] .
\end{align*}
Then
\[
 \max_{\|u\|_2=1}\left\| \frac{\partial r^{\pi_{\theta_\beta}}}{\partial \beta} \right\|_\infty \le R\sqrt{\cA} .\]
Similarly,
\begin{align*}
\left\| \frac{\partial^2 r^{\pi_{\theta_\beta}}}{\partial \beta^2} \right\|_\infty
&\le \max_s R \sum_a \left[ \frac{\partial^2 \pi_{\theta_\beta}(a|s)}{\partial \beta^2} \Big|_{\beta=0} \right] \\
&= 0. 
\end{align*}

Derivative of value state function with respect to $\beta$ is 
\[
\frac{\partial V^{\pi_{\theta_\beta}}(s)}{\partial \beta}
=  e_s^\top M^{\pi_{\theta_\beta}} \frac{\partial T^{\pi_{\theta_\beta}} }{\partial \beta} M^{\pi_{\theta_\beta}} r^{\pi_{\theta_\beta}}
+ e_s^\top M^{\pi_{\theta_\beta}} \frac{\partial r^{\pi_{\theta_\beta}}}{\partial \beta}.
\]
by matrix calculus $\frac{\partial A^{-1}}{\partial \theta} = A^{-1}\frac{\partial A}{\partial \theta}A^{-1}$.
Taking second derivative w.r.t. $\beta$,
\begin{align*}
\frac{\partial^2 V^{\pi_{\theta_\beta}}(s)}{\partial \beta^2}
&= 2\cdot e_s^\top M^{\pi_{\theta_\beta}} \frac{\partial T^{\pi_{\theta_\beta}} }{\partial \beta} 
   M^{\pi_{\theta_\beta}} \frac{\partial T^{\pi_{\theta_\beta}} }{\partial \beta} M^{\pi_{\theta_\beta}} r^{\pi_{\theta_\beta}} + e_s^\top M^{\pi_{\theta_\beta}} \frac{\partial^2 T^{\pi_{\theta_\beta}} }{\partial \beta^2} M^{\pi_{\theta_\beta}} r^{\pi_{\theta_\beta}} \\
&\quad + 2\cdot e_s^\top M^{\pi_{\theta_\beta}} \frac{\partial T^{\pi_{\theta_\beta}} }{\partial \beta} M^{\pi_{\theta_\beta}} \frac{\partial r^{\pi_{\theta_\beta}}}{\partial \beta} 
+ e_s^\top M^{\pi_{\theta_\beta}} \frac{\partial^2 r^{\pi_{\theta_\beta}}}{\partial \beta^2}.
\end{align*}

For the last term,
\begin{align*}
\left| e_s^\top M^{\pi_{\theta_\beta}} \frac{\partial^2 r^{\pi_{\theta_\beta}}}{\partial \beta^2} \Big|_{\beta=0} \right|
= 0. 
\end{align*}

For the second last term,
\begin{align*}
 \max_{\|u\|_2=1}\left| e_s^\top M^{\pi_{\theta_\beta}} \frac{\partial T^{\pi_{\theta_\beta}} }{\partial \beta} M^{\pi_{\theta_\beta}} \frac{\partial r^{\pi_{\theta_\beta}}}{\partial \beta} \Big|_{\beta=0} \right|
&\le \max_{\|u\|_2=1}\infn{ M^{\pi_{\theta_\beta}} \frac{\partial T^{\pi_{\theta_\beta}} }{\partial \beta} M^{\pi_{\theta_\beta}} \frac{\partial r^{\pi_{\theta_\beta}}}{\partial \beta} \Big|_{\beta=0} } \\
&\le C_\beta \max_{\|u\|_2=1}\left\| \frac{\partial T^{\pi_{\theta_\beta}} }{\partial \beta} M^{\pi_{\theta_\beta}} \frac{\partial r^{\pi_{\theta_\beta}}}{\partial \beta} \Big|_{\beta=0} \right\|_\infty 
\\
&\le C_\beta\sqrt{|\cA|} \max_{\|u\|_2=1}\| M^{\pi_{\theta_\beta}} \tfrac{\partial r^{\pi_{\theta_\beta}}}{\partial \beta} \Big|_{\beta=0} \|_\infty 
\\
&\le C_\beta^2\sqrt{|\cA|} \max_{\|u\|_2=1}\left\| \tfrac{\partial r^{\pi_{\theta_\beta}}}{\partial \beta} \Big|_{\beta=0} \right\|_\infty 
 \\
&\le C_\beta^2|\cA| R. 
\end{align*}

For the first term, similarly,
\begin{align*}
\max_{\|u\|_2=1}\left| e_s^\top M^{\pi_{\theta_\beta}} \frac{\partial T^{\pi_{\theta_\beta}} }{\partial \beta} M^{\pi_{\theta_\beta}} \frac{\partial T^{\pi_{\theta_\beta}} }{\partial \beta} M^{\pi_{\theta_\beta}} r^{\pi_{\theta_\beta}} \Big|_{\beta=0} \right|
&\le C_\beta\cdot  \sqrt{\cA }  \cdot C_\beta \sqrt{\cA } \cdot C_\beta \cdot R \\
&= C_\beta^3 R |\cA| .
\end{align*}

For the second term,
\begin{align*}
\left| e_s^\top M^{\pi_{\theta_\beta}} \frac{\partial^2 P^{\pi_{\theta_\beta}} }{\partial \beta^2} M^{\pi_{\theta_\beta}} r^{\pi_{\theta_\beta}} \Big|_{\beta=0} \right|
= 0
\end{align*}

Finally, we have
\begin{align*}
\left| \frac{\partial^2 V^{\pi_{\theta_\beta}}(s)}{\partial \beta^2} \Big|_{\beta=0} \right|
&\le 2\cdot \left| e_s^\top M^{\pi_{\theta_\beta}} \frac{\partial T^{\pi_{\theta_\beta}} }{\partial \beta} M^{\pi_{\theta_\beta}} \frac{\partial T^{\pi_{\theta_\beta}} }{\partial \beta} M^{\pi_{\theta_\beta}} r^{\pi_{\theta_\beta}} \Big|_{\beta=0} \right| \\
&\quad +  \left| e_s^\top M^{\pi_{\theta_\beta}} \frac{\partial^2 T^{\pi_{\theta_\beta}} }{\partial \beta^2} M^{\pi_{\theta_\beta}} r^{\pi_{\theta_\beta}} \Big|_{\beta=0} \right| \\
&\quad + 2 \left| e_s^\top M^{\pi_{\theta_\beta}} \frac{\partial T^{\pi_{\theta_\beta}} }{\partial \beta} M^{\pi_{\theta_\beta}} \frac{\partial r^{\pi_{\theta_\beta}}}{\partial \beta} \Big|_{\beta=0} \right| 
+ \left| e_s^\top M^{\pi_{\theta_\beta}} \frac{\partial^2 r^{\pi_{\theta_\beta}}}{\partial \beta^2} \Big|_{\beta=0} \right| \\
&\le 2 C_\beta^2 (C_\beta+1)R |\cA| .
\end{align*}
\end{proof}

We now prove another key lemma for Theorem \ref{thm:proj_gd}.
\begin{lemma}[Gradient domination]\label{lem:grad_dom} Under Assumption \ref{assump_total}, for $\pi \in \Pi_{\alpha}$,
    \[V_\mu^{\pi_{\alpha}^\star} -V_\mu^\pi \le \left\| \frac{\delta^{\pi^\star}_\mu}{\mu} \right\|_\infty 
\max_{\bar{\pi} \in \Pi_{\alpha}} (\bar{\pi} - \pi)^\top \nabla_\pi V^\pi_\mu.\]
\end{lemma}
\begin{proof}
We basically follow the proof strategy of \cite{agarwal2021theory}. Let $A^\pi(s,a)= Q^\pi(s,a)-V^\pi(s)$. Then, we have 
    \begin{align*}
V^{\pi_{\alpha}^\star}_\mu - V^\pi_\mu 
&=  \sum_{s,a} \delta^{\pi^\star}_\mu(s)\pi^\star(a|s)A^\pi(s,a) \\
&\le \sum_{s} \delta^{\pi^\star}_\mu(s)\max_{\bar{\pi} \in \Pi_{\alpha}}\left( \sum _a  \bar{\pi} (a \,|\, s)  A^\pi(s,a)\right) \\
&\le  \left( \max_s \frac{\delta^{\pi^\star}_\mu(s)}{\delta^\pi_\mu(s)} \right) \sum_{s} \delta^\pi_\mu(s) \max_{\bar{\pi} \in \Pi_{\alpha}}\left( \sum _a  \bar{\pi} (a \,|\, s)  A^\pi(s,a)\right),
\end{align*}
\begin{equation}
\end{equation}
where first equality comes from Corollary \ref{cor:pdl}, the last inequality follows since $\max_{\bar{\pi} \in \Pi_{\alpha}}\left( \sum _a  \bar{\pi} (a \,|\, s)  A^\pi(s,a)\right) \ge 0$ for all $s \in \cS$ and policies $\pi \in \Pi_\alpha$.  Also, we have

\begin{align*}
\sum_{s} \delta^\pi_\mu(s) \max_{\bar{\pi} \in \Pi_{\alpha}}\left( \sum _a  \bar{\pi} (a \,|\, s)  A^\pi(s,a)\right)
&= \max_{\bar{\pi} \in \Pi_{\alpha}} \sum_{s,a} \delta^\pi_\mu(s) \bar{\pi}(a|s)  A^\pi(s,a) 
\\&= \max_{\bar{\pi} \in \Pi_{\alpha}} \sum_{s,a} \delta^\pi_\mu(s)\big( \bar{\pi}(a|s) - \pi(a|s)\big) A^\pi(s,a) \\
&= \max_{\bar{\pi} \in \Pi_{\alpha}} \sum_{s,a} \delta^\pi_\mu(s)\big( \bar{\pi}(a|s) - \pi(a|s)\big) Q^\pi(s,a) \\
&= \max_{\bar{\pi} \in \Pi_{\alpha}} (\bar{\pi} - \pi)^\top \nabla_\pi V^\pi_\mu.
\end{align*}
where the first equality follows from the fact that $\max_{\bar{\pi}}$ is attained at an action which maximizes 
$A^\pi(s,\cdot)$, the second equality is from 
$\sum_a \pi(a|s) A^\pi(s,a) = 0$, the third equality is from 
$\sum_a (\bar{\pi}(a|s) - \pi(a|s)) V^\pi(s) = 0$ for all $s$, 
and the last equality follows from the Theorem \ref{thm::policy_gd} with direct parameterization. 
Finally,
\begin{align*}
V^{\pi^\star_\alpha}_\mu - V^\pi_\mu 
&\le \left\| \frac{\delta^{\pi^\star}_\mu}{\delta^\pi_\mu} \right\|_\infty 
\max_{\bar{\pi} \in \Pi_{\alpha}} (\bar{\pi} - \pi)^\top \nabla_\pi V^\pi_\mu \\
&\le  \left\| \frac{\delta^{\pi^\star}_\mu}{\mu} \right\|_\infty 
\max_{\bar{\pi} \in \Pi_{\alpha}} (\bar{\pi} - \pi)^\top \nabla_\pi V^\pi_\mu.
\end{align*}
where the last step follows from 
$\max_{\bar{\pi} \in \Pi_{\alpha}} (\bar{\pi} - \pi)^\top \nabla_\pi V^\pi_\mu \ge 0$ 
for any policy $\pi$ and $\delta^\pi_\mu(s) \ge \mu(s)$.
\end{proof}

\subsection{Proof of Theorem \ref{thm:proj_gd}}
Following the proof strategy of \cite{xiao2022convergence}, we consider composite optimization problem:
   \[\text{min}_{x \in \real^n} F(x) :=f(x)+ \Psi (x)\]
where $f$ is $L$-smooth and $\Psi$ is proper, convex, and closed \citep{rockafellar1970}. Define $F^\star = \min_x F(x)$.

For convex function $\phi$, define proximal operator as 
\[\textbf{prox}_{\phi}(x) = \text{argmin}_y \{\phi (y)+\frac{1}{2}\|y-x\|^2_2\}.\]
Then, for composite optimization problem, proximal gradient method is
\[x^{k+1} = \textbf{prox}_{\eta_k\Psi}(x^k-\eta_k \nabla f(x^k)).\]
Specifically, let $\eta_k=\frac{1}{L} $, and define
\[T_{L}(x) = \textbf{prox}_{\frac{1}{L}\Psi}(x-\frac{1}{L} \nabla f(x))\]
such that proximal gradient method can be expressed as $x^{k+1} = T_L(x^k)$. We also define gradient mapping
\[G_L=L(x-T_L(x)).\]
\begin{definition}[weak gradient-mapping domination]
Consider composite optimization problem.
We say that \(F\) satisfies a weak gradient-mapping dominance condition, of exponent
\(\alpha \in (1/2,1]\), if there exists \(\omega>0\) such that
\[
\|G_L(x)\|_2 \;\ge\; \sqrt{2\omega}\,\big(F(T_L(x)) - F^\star\big)^{\alpha},
\qquad \forall\, x \in \operatorname{dom}\Psi.
\]
\end{definition}

\begin{fact}\citep[Theorem 4]{xiao2022convergence}\label{fact:xiao}
Consider the composite optimization problem.
Suppose \(F\) is weakly gradient-mapping dominant with exponent
\(\alpha =1\) and parameter \(\omega\).
Then, for all \(k \ge 0\), the proximal gradient method with step size \(\eta_k = 1/L\)
generates a sequence \(\{x^{k}\}\) that satisfies
\[
F\!\left(x^{k}\right) - F^\star
\;\le\;
\max\!\left\{
\frac{8L}{\omega\,k},
\ \left(\frac{\sqrt{2}}{2}\right)^{\!k}
\left(F\!\left(x^{0}\right) - F^\star\right)
\right\}.
\]
\end{fact}

\begin{fact}\citep[Theorem 1]{nesterov2013gradient}\label{fact:nesterov} Consider the composite optimization problem where $f$ is $L$-smooth on closed convex set $C$ and $\Psi$ is the indicator function with $C$. Then, for $x,y \in C$,
\[
\begin{aligned}
\langle \partial F(T_L(y)),\,  T_L(y)-x  \rangle
&\le
  2 
   \| G_L(y) \|_{2}\, \| T_L(y) - x \|_2 .
\end{aligned}
\]
  
\end{fact} 
Note that if $\Psi$ is the indicator function of a convex closed non empty subset of $C \in \real^n$, $\Psi$ is closed, convex, and proper and  $\textbf{prox}_{\eta \Psi} = \textbf{proj}_C$ where $\textbf{proj}$ is projection operator. 

Now, we apply previous results to our projected policy gradient setup. Let $-\Psi$ be indicator function with $\Pi_{\alpha}$ and  $f(\pi) = -V^\pi_\mu$. Then $\pi^{k+1} = T_L(\pi^k)$ is projected policy gradient.

To prove Theorem \ref{thm:proj_gd}, we first prove following lemma. 
\begin{lemma}\label{lem:grad_map}
Under Assumption \ref{assump_total}, for a given $\mu$ with full support, suppose that
\(V^\pi_\mu\) is \(L\)-smooth for $\pi \in \Pi_{\alpha}$. Then,
\[
V_\mu^{\pi^\star_\alpha}-V_\mu^{T_L(\pi)}\le\;
2\sqrt{2|\mathcal S|}\,
\Big\|\tfrac{\delta_\mu^{\pi^\star_\alpha}}{\mu}\Big\|_\infty\,
\big\|G_L(\pi)\big\|_2.
\]
\end{lemma}

\begin{proof}
By Fact \ref{fact:nesterov}, for all
\(\pi,\pi'\in\Pi_{\alpha}\),
\[
\big\langle \nabla V_\mu^{T_L(\pi)},
           \,\pi'-T_L(\pi)\big\rangle
\;\le\;
2\,\|G_L(\pi)\|_2\,\|T_L(\pi)-\pi'\|_2 .
\]
Since \(T_L(\pi)\in\Pi_{\alpha}\) and
\(\|\pi_1-\pi_2\|_2 \le \sqrt{2|\mathcal S|}\) for any \(\pi_1,\pi_2\in\Pi_{\alpha}\),
we obtain
\[
\max_{\pi'\in\Pi_{\alpha}}
\big\langle \nabla V_\mu^{T_L(\pi)},\,\pi'-T_L(\pi)\big\rangle
\;\le\; 2\sqrt{2|\mathcal S|}\,\|G_L(\pi)\|_2 .
\]
Combining the above bound with Lemma \ref{lem:grad_dom} gives the desired inequality.
\end{proof} 
Now we are ready to prove Theorem \ref{thm:proj_gd}.
\begin{proof}
By lemma \ref{lem:grad_map}, weak gradient-mapping domination condition holds with
\[
\alpha = 1, \qquad
\omega = \frac{1}{16\,|\mathcal S|}\,
\left\|\frac{\delta_\mu^{\pi_{\alpha}^\star}}{\mu}\right\|_\infty^{-2},
\qquad
L = 2 R C_\alpha^2(C_\alpha+1)|\mathcal A|.
\]

Then, by applying Fact \ref{fact:xiao}, we get
\[
V_\mu^{\pi^\star_\alpha}-V_\mu^{\pi_{k}}  \le \frac{8L}{\omega k} = \frac{256 R|\mathcal S|\,|\mathcal A|C_\alpha^2(C_\alpha+1)}{k }\,
\left\|\frac{\delta_\mu^{\pi_{\alpha}^\star}}{\mu}\right\|_\infty^{2}.
\]
Lastly, in Fact \ref{fact:xiao}, exponential decay part 
is always smaller than the sublinear part.
\end{proof}
\subsection{Proof of Corollary \ref{cor:proj_grad}}
\begin{proof}
   Let $\alpha=\frac{\epsilon}{2|\cS||\cA| \infn{\delta^{\pi^\star}_{\mu}} \infn{Q^{\pi^\star}}} $. Since there exist $\pi \in \Pi_{\alpha}$ such that $\infn{\pi^\star-\pi} \le \alpha$, by Corollary \ref{cor:pdl_reward}, we have
   \begin{align*}
V_\mu^{\star} - V_\mu^{\pi}
&=  \sum_{s' \in \cS}
\sum_{a \in \cA}
Q^{\pi}(s',a)(\pi^\star(a \mid s')-\pi(a \mid s')) 
 \delta^{\pi^\star}_\mu(s')
 \\& \le |\cS||\cA|\infn{\delta^{\pi^\star}_{\mu}} \infn{Q^{\pi}} \alpha
 \\& \le \frac{\epsilon}{2}
\end{align*}
   and this implies 
   $V^{\star}_\mu\;-\;V^{\pi^\star_{\alpha}}_\mu  \le \frac{\epsilon}{2}.$ and the result comes from Theorem \ref{thm:proj_gd} by having $ V^{\pi^\star_{\alpha}}_\mu-V^{\pi_k}_\mu \le \frac{\epsilon}{2}$.
\end{proof}

\section{Omitted proofs in Section \ref{sec:natural_gd}}\label{appen:miss_6}

\subsection{Reformulation of natural policy gradient}
We basically follow the derivation in \cite{agarwal2021theory}. First, we provide explicit form of policy gradient with softmax parametrization. 
\begin{lemma}\label{pg_softmax}
\[
\frac{\partial V^{\pi_\theta}_\mu}{\partial \theta_{s,a}}
    = \delta_\mu^{\pi}(s)\pi_\theta(a\mid s)(Q^{\pi_\theta}(s,a)-V^{\pi_\theta}(s)).
\]    
\end{lemma} 
\begin{proof}
    With the softmax policy parameterization, we have 
\[
\frac{\partial \log \pi_{\theta}(a \mid s)}{\partial \theta_{s',a'}}
= \mathbf{1}[s = s'] \left( \mathbf{1}[a = a'] - \pi_{\theta}(a' \mid s) \right)
\]
where $\mathbf{1}$ is the indicator function. Then, by Theorem \ref{thm::policy_gd},
\begin{align*}
\frac{\partial V^{\pi_\theta}_\mu}{\partial \theta_{s',a'}}
&=  \sum_{s \in \cS}\delta_\mu^{\pi}(s')
    \mathbb{E}_{a \sim \pi_\theta(\cdot\mid s)}
    \biggl[ Q^{\pi_\theta}(s,a)\,
           \frac{\partial \log \pi_\theta(a\mid s)}{\partial \theta_{s',a'}}
    \biggr] \\
&=  \sum_{s \in \cS}\delta_\mu^{\pi}(s')
    \mathbb{E}_{a \sim \pi_\theta(\cdot\mid s)}
    \bigl[ Q^{\pi_\theta}(s,a)\,
           \mathbf{1}[s=s']\bigl(\mathbf{1}[a=a']-\pi_\theta(a'\mid s)\bigr)
    \bigr] \\
&=  \delta_\mu^{\pi}(s')
    \mathbb{E}_{a \sim \pi_\theta(\cdot\mid s')}
    \bigl[ Q^{\pi_\theta}(s',a)\bigl(\mathbf{1}[a=a']-\pi_\theta(a'\mid s')\bigr)
    \bigr] \\
&= \delta_\mu^{\pi}(s')
    \Bigl(
      \mathbb{E}_{a \sim \pi_\theta(\cdot\mid s')}
        \bigl[Q^{\pi_\theta}(s',a)\mathbf{1}[a=a']\bigr]
      - \pi_\theta(a'\mid s')
        \mathbb{E}_{a \sim \pi_\theta(\cdot\mid s')}
        \bigl[Q^{\pi_\theta}(s',a)\bigr]
    \Bigr) \\
&=  \delta_\mu^{\pi}(s')\,\pi_\theta(a'\mid s')(Q^{\pi_\theta}(s',a') - V^{\pi_\theta}(s')).
\end{align*}
\end{proof}

Now, we derive the reformulation of natural policy gradient. 
\begin{proof}
First, we have
\[
w^\top \nabla_\theta \log \pi_\theta(a\mid s)
= w_{s,a} - \sum_{a'\in\mathcal{A}} w_{s,a'} \pi_\theta(a'\mid s).
\]
Let \(\overline{w}_s= \sum_{a'\in\mathcal{A}} w_{s,a'} \pi_\theta(a'\mid s)\), which is independent of \(a\).
Then,
\begin{align*}
\mathcal{F}_\mu(\theta) w
&= \sum_{s\in \cS}\delta_{\mu}^{\pi_\theta}(s)
   \mathbb{E}_{a\sim \pi_\theta(\cdot\mid s)}
   \bigl[
      \nabla_\theta \log \pi_\theta(a\mid s)
      \bigl( w^\top \nabla_\theta \log \pi_\theta(a\mid s) \bigr)
   \bigr] \\
&= \sum_{s\in \cS}\delta_{\mu}^{\pi_\theta}(s)
   \mathbb{E}_{a\sim \pi_\theta(\cdot\mid s)}
   \bigl[
      \nabla_\theta \log \pi_\theta(a\mid s)
      \bigl( w_{s,a} - \overline{w}_s \bigr)
   \bigr] \\
&= \sum_{s\in \cS}\delta_{\mu}^{\pi_\theta}(s)
   \mathbb{E}_{a\sim \pi_\theta(\cdot\mid s)}
   \bigl[
      w_{s,a}\,\nabla_\theta \log \pi_\theta(a\mid s)
   \bigr],
\end{align*}
where the last equality uses log derivative trick with \(\overline{w}_s\), and this implies that
\[
\bigl[\mathcal{F}_\mu(\theta) w\bigr]_{s',a'}
= \delta_{\mu}^{\pi_\theta}(s')\,\pi_\theta(a'\mid s')
  \bigl( w_{s',a'} - \overline{w}_{s'} \bigr).
\]
By property of Moore--Penrose pseudoinverse, $
\bigl(\mathcal{F}_\mu(\theta)\bigr)^{\dagger}\nabla V^{\pi_\theta}_\mu$ is the minimum-norm solution of $\min_{w}\bigl\| \nabla V^{\pi_\theta}_\mu - \mathcal{F}_\mu(\theta) w \bigr\|_2^2 $ where $\|\cdot\|_2$ is Euclidean norm. By lemma
\ref{pg_softmax}, we have 
\begin{align*}
\bigl\| \nabla V^{\pi_\theta}_\mu - \mathcal{F}_\mu(\theta)w \bigr\|^2
&= \sum_{s,a} \left(  \delta_{\mu}^{\pi_\theta}(s)\pi_\theta(a\mid s)
\left( Q^{\pi_\theta}(s,a) - V^{\pi_\theta}(s)
      - w_{s,a} - \sum_{a'\in \mathcal{A}} w_{s,a'} \pi_\theta(a'\mid s)
\right) \right)^{2}.
\end{align*}
If \( w = Q^{\pi_\theta}(s,a) - V^{\pi_\theta}(s) \), $\bigl\| \nabla V^{\pi_\theta}_\mu - \mathcal{F}_\mu(\theta) w \bigr\|^2=0$, and $w$ has the form of $Q^{\pi_\theta}(s,a) + v(s)$ where \(v\) only depends on state. Therefore, plugging $
 \mathcal{F}_\mu(\theta)^{\dagger}\nabla V^{\pi_\theta}_\mu =Q^{\pi_\theta}(s,a) + v(s)$ into the original form of natural policy gradient, we obtain reformulation of natural policy gradient.


\end{proof}

Following the proof strategy of \cite{xiao2022convergence}, we consider mirror descent framework.

Let \(h:\cM(\cA)\to\mathbb{R}\) be a strictly convex function and continuously
differentiable on the (relative) interior of \(\cM(\cA)\), denoted as
\(\rint\,\cM(\cA)\).
Define Bregman divergence generated by \(h\) as
\[
D(p,p') \;=\; h(p)-h(p')-\langle \nabla h(p'),\, p-p' \rangle,
\]
 for any \(p\in\cM(\cA)\) and \(p'\in\rint\,\cM(\cA)\).
Specifically, Kullback--Leibler (KL) divergence, generated by the negative entropy $h(p)=\sum_{a\in\cA} p_a\log p_a$ formulated as $
  D(p,p')=\sum_{a\in\cA} p_a \log\!\frac{p_a}{p'_a}\, .
  $



For any \(\mu\in\cM(\cS)\), we define a weighted divergence function
\[
D_\mu(\pi,\pi') \;=\; \sum_{s\in\cS}\mu(s)\,D(\pi(\cdot\,|\,s),\pi'(\cdot\,|\,s)).
\]
Following the derivations of \cite{shani2020adaptive,xiao2022convergence}, we consider
policy mirror descent methods with dynamically weighted divergences:
\[
\pi_{k+1}
=\argmin_{\pi\in\Pi}
\left\{
-\eta_k\big\langle\nabla V_\mu^{\pi_{k}},\,\pi\big\rangle
+\,D_{\delta_\mu(\pi_{k})}\!\big(\pi,\pi_{k}\big)
\right\},
\]
where \(\eta_k\) is the step size, \(\mu\in\cM_+(\cS)\) is an arbitrary
state distribution. 

Consider direction parametrization, and  by Theorem \ref{thm::policy_gd}, we have
\begin{align*}
    \pi_{k+1}
&=\argmin_{\pi\in\Pi}
\left\{
-\eta_k \sum_{s \in \cS} \delta^{\pi_k}_{\mu}(s) \big(\sum_{s\in\cS} Q^{\pi_k}(s,a)\pi(a\,|\,s)+D\big(\pi(\cdot\,|\,s),\pi_{k}(\cdot\,|\,s)\big)
\right\}
\\&=\argmin_{\pi\in\Pi}
\left\{
-\eta_k \sum_{s \in \cS}  \big(\sum_{s\in\cS} Q^{\pi_k}(s,a)\pi(a\,|\,s)+D\big(\pi(\cdot\,|\,s),\pi_{k}(\cdot\,|\,s)\big)
\right\}
\end{align*}

and this is reduced to 
\[
\pi_{k+1}(\cdot \, |\, s)
=
\argmin_{p\in\cM(\cA)}
\Big\{
-\eta_k\ \sum_{a\in \cA}Q^{\pi_{k}}(a, s) p(a)
+ D\big(p,\,\pi_{k}(\cdot \,|\, s)\big)
\Big\}
\]
for all $ s\in\cS $. It was known that solution form of this update is natural policy gradient with softmax parameterization \citep[Section~9.1]{beck2017first}. 

We say function $h$ is of Legendre type if it is essentially smooth and strictly
convex in the (relative) interior of $\dom h$ and 
$h$ is essential smoothness if $h$ is differentiable and
$\|\nabla h(x_k)\|\to\infty$ for every sequence $\{x_k\}$ converging to a
boundary point of $\dom h$.

\begin{fact}[Three-point descent lemma]\citep[Lemma 6]{xiao2022convergence}\label{fact:bregman}
Suppose that $\,\mathcal{C}\subset\mathbb{R}^n$ is a closed convex set,
$\phi:\mathcal{C}\to\mathbb{R}$ is a proper, closed, and convex function,
$D(\cdot,\cdot)$ is the Bregman divergence generated by a function $h$ of
Legendre type and $\rint \dom h \cap \mathcal{C}\neq\varnothing$.
For any $x\in \rint \dom h$, let
\[
x^{+}=\argmin_{u\in\mathcal{C}} \big\{\, \phi(u)+D(u,x) \,\big\}.
\]
Then, $x^{+}\in \rint \dom h \cap \mathcal{C}$ and for any $u\in\mathcal{C}$,
\[
\phi(x^{+}) + D(x^{+},x) \;\le\; \phi(u) + D(u,x) - D(u,x^{+}).
\]
\end{fact}

In our setup, $\mathcal{C}=\cM(\cA)$,  $\phi(p)=-\eta_k \langle Q^{\pi_k}(\cdot,s), p(\cdot)\,\rangle$, and $h$ is the negative-entropy function, which is also
of Legendre type, satisfying $\rint\,\operatorname{dom}h \cap \mathcal{C}
= \rint\cM(\cA) = \rint \operatorname{dom}h$.
Therefore, if we start with an initial point in $\rint\cM(\cA)$, then every
iterate will stay in $\rint\cM(\cA)$.

\subsection{Proof of Lemma \ref{lem:ngd_descent}}
We proved more detailed version of  Lemma \ref{lem:ngd_descent}.
\begin{lemma}\label{lem:7'}
Under Assumption~\ref{assump_total}, for given $\pi_0\in \Pi_+$ and $\mu$ with full support, the natural policy gradient with step size $\eta_k>0$ generates a sequence of policies $\{\pi_{k}\}^\infty_{k=1}$ satisfying
\[ \sum_{a\in \cA} Q^{\pi_{k}}(a, s )  (\pi_{k}(a\,|\,s)-\pi_{k+1}(a\,|\,s))  \le 0,\qquad \forall\, s\in \cS ,\]and
\[
V_\mu^{\pi_k} \le V_\mu^{\pi_{k+1}}.
\]
\end{lemma}

\begin{proof}[Proof of Lemma \ref{lem:ngd_descent}]
    Applying Fact \ref{fact:bregman} with
$\mathcal{C}=\cM(\cA)$, $\phi(p)=-\eta_k \sum_{a\in\cA} Q^{\pi_k}(s,a)p(a)$, and KL divergence as Bregman divergence,
we obtain
\begin{align*}
    &\eta_k \sum_{a\in \cA} Q^{\pi_k}(s,a) p(a)
+\text{KL}\big(\pi_{k+1}(\cdot\,|\,s),\, \pi_k(\cdot\,|\,s)\big)
\\&\le
\eta_k \sum_{a\in \cA} Q^{\pi_k}(s,a) \pi_{k+1}(a\,|\,s)
+ \text{KL}\big(p, \pi_k(\cdot\,|\,s)\big) -\text{KL}\big(p, \pi_{k+1}(\cdot\,|\,s)\big)
\end{align*}
for any $p\in\cM(\cA)$. Rearranging terms and dividing both sides by $\eta_k$, we get
\begin{align*}
    &\sum_{a\in \cA} Q^{\pi_k}(s,a)( p(a)-\pi_{k+1}(a\,|\,s))
+ \frac{1}{\eta_k} \text{KL}\!\big(\pi_{k+1}(\cdot\,|\,s), \pi_k(\cdot\,|\,s)\big)
\\&\le\;
\frac{1}{\eta_k} \text{KL}\!\big(p, \pi_k(\cdot\,|\,s)\big)
- \frac{1}{\eta_k} \text{KL}\!\big(p, \pi_{k+1}(\cdot\,|\,s)\big) \tag{$*$}.
\end{align*}
Letting \(p=\pi_k(\cdot\,|\,s)\) in previous inequality yields
\[
\sum_{a\in \cA} Q^{\pi_k}(s,a)(\pi_{k}(a\,|\,s) - \pi_{k+1}(a\,|\,s)) \le
-\frac{1}{\eta_k} \text{KL}\big(\pi_{k+1}(\cdot\,|\,s), \pi_k(\cdot\,|\,s)\big)
-\frac{1}{\eta_k} \text{KL}\big(\pi_k(\cdot\,|\,s), \pi_{k+1}(\cdot\,|\,s)\big).
\]
Then, the first results comes from nonnegativity of Bregman divergence and second result comes from Corollary \ref{cor:pdl}.
\end{proof}


\subsection{Proof of Theorem \ref{thm:ngd}}
\begin{proof}
Consider previous inequality ($*$). Let \(p=\pi(\cdot \,|\, s)\in\Pi_+\)  and add--subtract \(\pi_{k}(\cdot \,|\, s)\) inside
the inner product. Then we have
\begin{align*}
    &\sum_{a\in \cA} Q^{\pi_k}(s,a)(\pi_{k}(a\,|\,s) - \pi_{k+1}(a\,|\,s))
+ \sum_{a\in \cA} Q^{\pi_k}(s,a)(\pi(a\,|\,s) - \pi_{k}(a\,|\,s)) \\&\le\; \frac{1}{\eta_k} \text{KL}\big(\pi (\cdot\,|\,s),\pi_{k}(\cdot \,|\, s)\big)
        - \frac{1}{\eta_k} \text{KL}\big(\pi (\cdot\,|\,s),\pi_{k+1} (\cdot\,|\,s)\big).
\end{align*}
This implies
\begin{align*}
    &\sum_{s \in \cS}\delta_\mu^{\pi}(s)\sum_{a\in \cA} Q^{\pi_k}(s,a)(\pi_{k}(a\,|\,s) - \pi_{k+1}(a\,|\,s))
+ \sum_{s \in \cS}\delta_\mu^{\pi}(s)\sum_{a\in \cA} Q^{\pi_k}(s,a)(\pi(a\,|\,s) - \pi_{k}(a\,|\,s))\\ & \le\; \frac{1}{\eta_k} \text{KL}_{\delta_\mu^{\pi}}(\pi, \pi_k) - \frac{1}{\eta_k} \text{KL}_{\delta_\mu^{\pi}}(\pi, \pi_{k+1}) .
\end{align*}
For the first term, 
\[
\begin{aligned}
&\sum_{s \in \cS}\frac{\delta_\mu^{\pi}(s)}{\|\delta_\mu^{\pi}\|_1}\sum_{a\in \cA} Q^{\pi_k}(s,a)(\pi_{k}(a\,|\,s) - \pi_{k+1}(a\,|\,s))
\\&\ge \sum_{s \in \cS}(\delta_\mu^{\pi})'(s)\sum_{a\in \cA} Q^{\pi_k}(s,a)(\pi_{k}(a\,|\,s) - \pi_{k+1}(a\,|\,s)) \\
&\ge \sum_{s \in \cS}\delta_{(\delta_\mu^{\pi})'}^{\pi_{k+1}}(s)\sum_{a\in \cA} Q^{\pi_k}(s,a)(\pi_{k}(a\,|\,s) - \pi_{k+1}(a\,|\,s)) \\
&= V_{(\delta_\mu^{\pi})'}^{\pi_{k}}-V_{(\delta_\mu^{\pi})'}^{\pi_{k+1}},
\end{aligned}
\]
where $(\delta_{\mu}^{\pi})'$ is probability distribution satisfying $(\delta_{\mu}^{\pi})'(s) = \frac{\delta_\mu^{\pi}(s)}{\infn{\delta_\mu^{\pi}}}$, the first and second inequalities come from Lemma \ref{lem:7'}, and the last equality comes from Corollary \ref{cor:pdl}.

For the second term, by  Corollary \ref{cor:pdl},
\[
\sum_{s \in \cS}\delta_\mu^{\pi}(s)\sum_{a\in \cA} Q^{\pi_k}(s,a)(\pi(a\,|\,s) - \pi_{k}(a\,|\,s))
= V_\mu^{\pi}- V_\mu^{\pi_{k}}  .
\]
Thus we have
\[
V_\mu^{\pi}-V_\mu^{\pi_{k}} 
\;\le\;\frac{1}{\eta_k} \text{KL}_{\delta_\mu^{\pi}}(\pi, \pi_k) - \frac{1}{\eta_k} \text{KL}_{\delta_\mu^{\pi}}(\pi, \pi_{k+1})
+ \infn{\delta_\mu^{\pi}}(V_{(\delta_\mu^{\pi})'}^{\pi_{k+1}}
 - V_{(\delta_\mu^{\pi})'}^{\pi_{k}}).
\]

Setting \(\eta_k=\eta\) for all \(k\ge 0\) and summing over \(k\) gives
\[
\sum_{i=0}^{k}\big(V_\mu^{\pi}- V_\mu^{\pi_{i}}  \big)
\;\le\;
\frac{1}{\eta} \text{KL}_{\delta_\mu^{\pi}}(\pi, \pi_0)  - \frac{1}{\eta} \text{KL}_{\delta_\mu^{\pi}}(\pi, \pi_{k+1}) 
+ \infn{\delta_\mu^{\pi}}(V_{(\delta_\mu^{\pi})'}^{\pi_{k+1}}
 - V_{(\delta_\mu^{\pi})'}^{\pi_{0}}).
\]

Since \(V_\mu^{\pi_{k}}\) are non-decreasing by Lemma \ref{lem:ngd_descent} and KL-divergence is non-negative, we conclude that
\[
V_\mu^\pi-V_\mu^{\pi_{k}}  
\;\le\; \frac{1}{k+1}\!\left( \frac{\text{KL}_{\delta_\mu^{\pi}}(\pi, \pi_0)}{\eta}
+ \infn{\delta_\mu^{\pi}} (\infn{V^\star_{+}}+\infn{V^{\pi_0}}) \right).
\]

For given $\eta, \epsilon>0$, there exist $\pi$ such that  $V^\star_{+,\mu}-V^\pi <\epsilon/2$ since $V^\star_{+,\mu} <\infty$. Then, by previous inequality, there exist $\pi_k$ such that $V^\pi-V^{\pi_k}_\mu <\epsilon/2$. Thus, we have $V^\star_{+,\mu}-V^{\pi_k}_\mu <\epsilon$. Since this holds for arbitrary $\epsilon$, we get $V^{\pi_k}_\mu \rightarrow V^\star_{+,\mu}.$

\end{proof}
\subsection{Proof of Corollary \ref{cor:ngd}}
\begin{proof}
In the previous proof of Theorem \ref{thm:ngd}, let $\pi =\pi^\star$, and use Corollary \ref{cor:pdl_reward} instead of Corollary \ref{cor:pdl} and the fact that $V^\pi_{\mu} \ge 0$. Then, we obtain the desired result.
\end{proof}
\subsection{Proof of Theorem \ref{thm:ngd_linear}}
Define
$
U^\pi_k=V_\mu^{\pi}- V_\mu^{\pi_{k}},
$ and the per-iteration distribution mismatch coefficient
\[
\vartheta^\pi_k \;:=\; \left\lVert \frac{\delta_\mu^{\pi}}{\delta_\mu^{\pi_{k}}} \right\rVert_{\infty}.
\]
We first prove following key lemma. 
\begin{lemma}\label{lem:ngd_linear}
Under Assumption~\ref{assump_total}, for a given $\mu$ with full support, the natural policy gradient with step size $\eta_k>0$ generates a sequence of policies $\{\pi_{k}\}^\infty_{k=1}$ satisfying,
\[
\vartheta^\pi_{k+1}\,\big(U^\pi_{k+1}-U^\pi_k\big) + U^\pi_k
\;\le\;
\frac{1}{\eta_k} \text{KL}_{\delta_\mu^{\pi}}(\pi, \pi_k) - \frac{1}{\eta_k} \text{KL}_{\delta_\mu^{\pi}}(\pi, \pi_{k+1}) 
\]
\end{lemma}


\begin{proof}[Proof of Lemma \ref{lem:ngd_linear}]
In the previous proof of Theorem \ref{thm:ngd}, we showed that 
\begin{align*}
    &\sum_{s \in \cS}\delta_\mu^{\pi}(s)\sum_{a\in \cA} Q^{\pi_k}(s,a)(\pi_{k}(a\,|\,s) - \pi_{k+1}(a\,|\,s))
+ \sum_{s \in \cS}\delta_\mu^{\pi}(s)\sum_{a\in \cA} Q^{\pi_k}(s,a)(\pi(a\,|\,s) - \pi_{k}(a\,|\,s))\\ & \le\; \frac{1}{\eta_k} \text{KL}_{\delta_\mu^{\pi}}(\pi, \pi_k) - \frac{1}{\eta_k} \text{KL}_{\delta_\mu^{\pi}}(\pi, \pi_{k+1}) .
\end{align*}
For the first term
\[
\begin{aligned}
&\sum_{s \in \cS}\delta_\mu^{\pi}(s)\sum_{a\in \cA} Q^{\pi_k}(s,a)(\pi_{k}(a\,|\,s) - \pi_{k+1}(a\,|\,s))\\
&= \sum_{s\in\cS}
   \frac{\delta_\mu^{\pi}}{\delta_\mu^{\pi_{k+1}}}\,
   \delta_\mu^{\pi_{k+1}}
   \sum_{a\in \cA} Q^{\pi_k}(s,a)(\pi_{k}(a\,|\,s) - \pi_{k+1}(a\,|\,s)) \\
&\ge
\left\lVert\frac{\delta_\mu^{\pi}}{\delta_\mu^{\pi_{k+1}}}\right\rVert_{\infty}
\sum_{s\in\cS} \delta_\mu^{\pi_{k+1}}
\sum_{a\in \cA} Q^{\pi_k}(s,a)(\pi_{k}(a\,|\,s) - \pi_{k+1}(a\,|\,s)) \\
&=
\left\lVert\frac{\delta_\mu^{\pi}}{\delta_\mu^{\pi_{k+1}}}\right\rVert_{\infty}
\Big(V_\mu^{\pi_{k}}-V_\mu^{\pi_{k+1}}\Big),
\end{aligned}
\]
where the first inequality is from Lemma \ref{lem:ngd_descent}, and the last equality comes from Corollary \ref{cor:pdl}. 

For the second term, by  Corollary \ref{cor:pdl},
\[
\sum_{s \in \cS}\delta_\mu^{\pi}(s)\sum_{a\in \cA} Q^{\pi_k}(s,a)(\pi(a\,|\,s) - \pi_{k}(a\,|\,s))
= V_\mu^{\pi}- V_\mu^{\pi_{k}}  .
\]
we obtain the desired result after substitution.

\end{proof}

We are now ready to prove Theorem \ref{thm:ngd_linear}
\begin{proof}
Since $\delta_{\mu}^{\pi_{k}}(s)\ge\mu(s)$ for all
$s\in \cS$, $\vartheta^\pi_k\le \vartheta^\pi_\mu$ for all $k\ge 0$. $U^\pi_{k+1}-U^\pi_{k}\le 0$ for all $k\ge 0$ by Lemma \ref{lem:7'}, and by Lemma \ref{lem:ngd_linear}, 
\[
\vartheta^\pi_\mu\,\big(U^\pi_{k+1}-U^\pi_k\big) + U^\pi_k
\;\le\;
\frac{1}{\eta_k} \text{KL}_{\delta_\mu^{\pi}}(\pi, \pi_k) - \frac{1}{\eta_k} \text{KL}_{\delta_\mu^{\pi}}(\pi, \pi_{k+1}). 
\]
Dividing both sides by $\vartheta^\pi_\mu$ and rearranging terms, we obtain
\[
U^\pi_{k+1}
+\frac{1}{\eta_k\vartheta^\pi_\mu} \text{KL}_{\delta_\mu^{\pi}}(\pi, \pi_{k+1}) 
\;\le\;
\left(1-\frac{1}{\vartheta^\pi_\mu}\right)
\left(
U^\pi_k + \frac{1}{\eta_k(\vartheta^\pi_\mu-1)} \text{KL}_{\delta_\mu^{\pi}}(\pi, \pi_{k}) 
\right).
\]
Since the step sizes satisfy condition,
$\eta_{k+1}(\vartheta^\pi_\mu-1)\ge \eta_k\vartheta^\pi_\mu >0$, we have
\[
U^\pi_{k+1}
+\frac{1}{\eta_{k+1}(\vartheta^\pi_\mu-1)} \text{KL}_{\delta_\mu^{\pi}}(\pi, \pi_{k+1})
\;\le\;
\left(1-\frac{1}{\vartheta^\pi_\mu}\right)
\left(
U^\pi_k + \frac{1}{\eta_k(\vartheta^\pi_\mu-1)} \text{KL}_{\delta_\mu^{\pi}}(\pi, \pi_{k})
\right).
\]
Therefore, by recursion,
\[
U^\pi_k
+\frac{1}{\eta_k(\vartheta^\pi_\mu-1)} \text{KL}_{\delta_\mu^{\pi}}(\pi, \pi_{k})
\;\le\;
\left(1-\frac{1}{\vartheta^\pi_\mu}\right)^{\!k}
\left(
U^\pi_0 + \frac{1}{\eta_0(\vartheta^\pi_\mu-1)} \text{KL}_{\delta_\mu^{\pi}}(\pi, \pi_{0})
\right).
\]
\end{proof}

\subsection{Proof of Corollary \ref{cor:ngd_linear}}
\begin{proof}
In the previous proof of Theorem \ref{thm:ngd_linear}, let $\pi =\pi^\star$ and use Corollary \ref{cor:pdl_reward} instead of Corollary \ref{cor:pdl}. Then, we obtain the desired result.
\end{proof}

\section{Omitted proofs in Section \ref{sec:SNPG}}\label{appen:D}
\subsection{Proof of Theorem \ref{thm:PE}}
First, we introduce Hoeffding inequality for sample complexity analysis. 
\begin{fact}[Hoeffding inequality]
      Let $X_1, \dots, X_n$ are independent random variables such that $ a_i \le X_i  \le b_i$ for all $i$. Then
\[ \mathbb{P}\left( \left|\frac{1}{n}\sum^n_{i=1} X_i- \expec \left[\frac{1}{n}\sum^n_{i=1} X_i\right]\right| \ge \epsilon\right) \le 2 \text{exp} \left(-\frac{2n^2 \epsilon^2}{\sum^n_{i=1}(b_i-a_i)^2}\right).\]
\end{fact}

\begin{proof}
Since
$\tilde{Q}_j(s_0,a_0)=\frac{1}{N}\sum^N_{j=1}\sum^{H}_{i=0}r_j(s_i,a_i),  
$
and  
\begin{align*}
   \expec_\pi \left[\frac{1}{N}\sum^N_{j=1} \sum^{H}_{i=0}r_j(s_i,a_i)\right] =\expec_\pi \left[\sum^{H}_{i=0}r_j(s_i,a_i)\right],
\end{align*}
we have
\begin{align*}
    \left|Q^\pi(s_0,a_0) -\expec_\pi \left[\sum^{H}_{i=0}r_j(s_i,a_i)\right]\right|&= \left| (P V^\pi+r)(s_0,a_0)-\left(P\left(\sum^{H-1}_{i=0}(T^\pi)^ir^\pi\right)+r\right)(s_0,a_0) \right| 
    \\&=\left|P\left((I-T^\pi)^{-1}-\sum^{H-1}_{i=0}(T^\pi)^i\right)r^\pi(s_0,a_0)\right|
   \\&=\left|P(I-T^\pi)^{-1}(T^\pi)^Hr^\pi(s_0,a_0)\right|
    \\&
    \le C_\pi R (1/2)^{\lfloor H/t_{\frac{1}{2},\pi}\rfloor }.
\end{align*}
 On the other hand, by Hoeffding inequality with $N= \frac{2(H+1)^2R^2}{ (\epsilon')^2}\log(\frac{2}{\delta})$, we get
\[\text{Prob} \left(\left|\tilde{Q}^\pi(s_0,a_0) - \expec_\pi \left[\frac{1}{N}\sum^N_{j=1} \sum^{H}_{i=0}r_j(s_i,a_i)\right]\right|  \le \epsilon' \right) \ge 1-\delta.\]
and by union bound over $(s_0,a_0) \in \cS \times \cA$, we have
$\infn{\tilde{Q}^\pi-\expec \left[ \tilde{Q}^\pi \right]}< \epsilon'$ with  $N= \frac{2(H+1)^2R^2|\cS||\cA|}{ (\epsilon')^2}\log(\frac{2|\cS||\cA|}{\delta})$.
Let $H =t_{\frac{1}{2},\pi} \log \left( 2C_\alpha R/\epsilon \right)$ and $\epsilon'=\epsilon/2$,. Then, by triangular inequality $\infn{Q^\pi-\tilde{Q}\pi} \le \infn{Q^\pi-\expec \left[\tilde{Q}^\pi\right]} +\infn{\expec \left[\tilde{Q}^\pi\right]-\tilde{Q}^\pi} \le \epsilon$ with sample complexity of  
\[NH|\cS||\cA|=\frac{8 t_{\frac{1}{2},\pi}^3 R^2 |\cS||\cA|}{\epsilon^2}  \log^3 \left( 2C_\pi R/\epsilon \right)\log\left(\frac{2|\cS||\cA|}{\delta}\right).\]
\end{proof}

\subsection{Proof of Theorem \ref{thm:SNPG} }
First, we consider \emph{inexact natural policy gradient}
\begin{align*}
\pi_{k+1}(\cdot \,|\, s)
=
\argmin_{p \in \cM(\mathcal{A})}
\left\{
-\eta_k \sum_{a\in \cS} \Tilde{Q}^{\pi_k}(s,a), p(a)
+ \text{KL}(p,\pi_k(\cdot \,|\, s))
\right\},
\qquad \forall s \in \mathcal{S},
\end{align*}
where $\Tilde{Q}^{\pi_k}$ is an inexact evaluation of $Q^{\pi_k}$.
We first study the convergence properties of inexact natural policy gradient under the following
assumption.

\begin{assumption}\label{assump:inexact}
The sequence of inexact evaluations $\left\{\Tilde{Q}^{\pi_k}\right\}^\infty_{k=1}$ satisfy 
$\left\|
\Tilde{Q}^{\pi_k} - Q^{\pi_k}
\right\|_\infty
\le \epsilon.
$
\end{assumption}
We first prove the key Lemma, counterpart of Lemma \ref{lem:ngd_descent}.
\begin{lemma}\label{inexact:mon}
Under Assumption~\ref{assump_total} and \ref{assump:inexact}, for given $\pi_0\in \Pi_+$ and $\mu$ with full support, the inexact natural policy gradient with step size $\eta_k>0$ generates a sequence of policies $\{\pi_{k}\}^\infty_{k=1}$ satisfying
\begin{align*}
\sum_{s\in\cS}\Tilde{Q}^{\pi_k}(s,a)(\pi_{k}(\cdot \,|\, s)-\pi_{k+1}(\cdot \,|\, s)) 
\le 0,
\qquad \forall s \in \mathcal{S},
\end{align*}
and
\begin{align*}
V_\mu^{\pi_k}-V_\mu^{\pi_{k+1}} 
\le
2  C_{\pi_k}\epsilon .
\end{align*}
\end{lemma}

\begin{proof}
The first inequality directly follows from the same arguments as in Lemma \ref{lem:7'}. By Corollary \ref{cor:pdl},
\begin{align*}
 V_\mu^{\pi_k}- V_\mu^{\pi_{k+1}} 
&=
\sum_{s \in \cS}\delta_\mu^{\pi_k}(s)\sum_{a\in \cA} Q^{\pi_k}(s,a)(\pi_{k}(a\,|\,s) - \pi_{k+1}(a\,|\,s))
\\
&=
\sum_{s \in \cS}\delta_\mu^{\pi_k}(s)\sum_{a\in \cA} \Tilde{Q}^{\pi_k}(s,a)(\pi_{k}(a\,|\,s) - \pi_{k+1}(a\,|\,s))
\\
&\quad
+
\sum_{s \in \cS}\delta_\mu^{\pi_k}(s)\sum_{a\in \cA} (Q^{\pi_k}(s,a)-\Tilde{Q}^{\pi_k}(s,a))(\pi_{k}(a\,|\,s) - \pi_{k+1}(a\,|\,s)).
\end{align*}

For the second term,  
\begin{align*}
\sum_{a\in \cA} (Q^{\pi_k}(s,a)-\Tilde{Q}^{\pi_k}(s,a))(\pi_{k}(a\,|\,s) - \pi_{k+1}(a\,|\,s))
&\le
\left\|
Q^{\pi_k} - \Tilde{Q}^{\pi_k}
\right\|_\infty
\left\|
\pi_{k}(a\,|\,s) - \pi_{k+1}(a\,|\,s)
\right\|_1
\nonumber\\
&\le
2 \left\|
\Tilde{Q}^{\pi_k} - Q^{\pi_k}
\right\|_\infty
\nonumber\\
&\le
2\epsilon ,
\end{align*}
where the first inequality is from Holder's inequality and the last inequality follows from Assumption \ref{assump:inexact}. Since first term is non-negatvie, we get the desired result. 
\end{proof}

Following will be used in the proof of Theorem \ref{thm:inexact}. 
\begin{fact}\label{fact:seq}[\cite{xiao2022convergence}]
Suppose $0 < \alpha < 1$, $b > 0$, and a nonnegative sequence
$\{a_k\}$ satisfies
\begin{align*}
a_{k+1} \le \alpha a_k + b,
\qquad \forall k \ge 0 .
\end{align*}
Then for all $k \ge 0$,
\begin{align*}
a_k \le \alpha^k a_0 + \frac{b}{1 - \alpha}.
\end{align*}
\end{fact}

Now, we present the convergence result of inexact natural policy gradient. 
\begin{theorem}\label{thm:inexact}
Under Assumption~\ref{assump_total} and \ref{assump:inexact}, for any $\pi_0, \pi \in \Pi_+$ and given $\mu$ with full support, the inexact natural policy gradient with constant step size $\eta>0$ generates a sequence of policies $\{\pi_{k}\}^\infty_{k=1}$ satisfying
\[V_\mu^\pi-V_\mu^{\pi_{k}}  
\;\le\; \frac{1}{k+1}\!\left( \frac{\text{KL}_{\delta_\mu^{\pi}}(\pi, \pi_0)}{\eta}
+ 2C_\pi\infn{V^\star_{+}} \right)+2 C_\pi(\vartheta^\pi_\mu+1) \epsilon+ \frac{\sum^{k-1}_{i=0}(i+1)C_{\pi_i}}{k+1}\epsilon.\]
and with adaptive step size $\eta_{k+1}(\vartheta^\pi_\mu - 1) \ge \eta_k \vartheta^\pi_\mu>0$  generates a sequence of policies $\{\pi_{k}\}^\infty_{k=1}$ satisfying 
\begin{align*}
V_\mu^\pi-V_\mu^{\pi_k}
\le
\left(1 - \frac{1}{\vartheta^\pi_\mu}\right)^k
\left(
V_\mu^\pi-V_\mu^{\pi_0}
+
\frac{\text{KL}_{\delta^\pi_\mu}(\pi, \pi_{0})}{\eta_0(\vartheta^\pi_\mu - 1)} 
\right)
+
4\vartheta^\pi_\mu C_k \epsilon .
\end{align*}
where $C_k= \max_{0\le i \le k}\{C_{\pi_{i}},C_\pi\}$.
\end{theorem}

\begin{proof}
Following the same arguments as in proof of Theorem \ref{thm:ngd},
\begin{align*}
    &\sum_{s \in \cS}\delta_\mu^{\pi}(s)\sum_{a\in \cA} \Tilde{Q}^{\pi_k}(s,a)(\pi_{k}(a\,|\,s) - \pi_{k+1}(a\,|\,s))
+ \sum_{s \in \cS}\delta_\mu^{\pi}(s)\sum_{a\in \cA} \Tilde{Q}^{\pi_k}(s,a)(\pi(a\,|\,s) - \pi_{k}(a\,|\,s))\\ & \le\; \frac{1}{\eta_k} \text{KL}_{\delta^\pi_\mu}(\pi, \pi_k) - \frac{1}{\eta_k} \text{KL}_{\delta^\pi_\mu}(\pi, \pi_{k+1}) .
\end{align*}

For the second term,
\begin{align*}
&\sum_{s \in \cS}\delta^\pi_\mu(s)\sum_{a\in \cA} \Tilde{Q}^{\pi_k}(s,a)(\pi(a\,|\,s) - \pi_{k}(a\,|\,s))
\\&\ge  \sum_{s \in \cS}\delta^{\pi}_\mu(s) \sum_{a\in \cA} Q^{\pi_k}(s,a)(\pi(a\,|\,s) - \pi_{k}(a\,|\,s)) \\&+  \sum_{s \in \cS}\delta^{\pi}_\mu(s) \sum_{a\in \cA} (\Tilde{Q}^{\pi_k}(s,a)-Q^{\pi_k}(s,a))(\pi(a\,|\,s) - \pi_{k}(a\,|\,s))
\\&\ge V_\mu^{\pi} - V_\mu^{\pi_k}   - 2  C_\pi\epsilon
\end{align*}

First, consider constant step size. For the first term, 
\[
\begin{aligned}
&\|\delta_\mu^{\pi}\|_1\sum_{s \in \cS}\frac{\delta_\mu^{\pi}(s)}{\|\delta_\mu^{\pi}\|_1}\sum_{a\in \cA} Q^{\pi_k}(s,a)(\pi_{k}(a\,|\,s) - \pi_{k+1}(a\,|\,s))\\&+  \sum_{s \in \cS}\delta^{\pi}_\mu(s) \sum_{a\in \cA} (\Tilde{Q}^{\pi_k}(s,a)-Q^{\pi_k}(s,a))(\pi(a\,|\,s) - \pi_{k}(a\,|\,s))
\\&\ge \|\delta_\mu^{\pi}\|_1 \sum_{s \in \cS}\delta_{(\delta_\mu^{\pi})'}^{\pi_{k+1}}(s)\sum_{a\in \cA} Q^{\pi_k}(s,a)(\pi_{k}(a\,|\,s) - \pi_{k+1}(a\,|\,s))- 2  C_\pi\epsilon 
\\
&= C_\pi(V_{(\delta_\mu^{\pi})'}^{\pi_{k}}-V_{(\delta_\mu^{\pi})'}^{\pi_{k+1}})- 2  C_\pi\epsilon,
\end{aligned}
\]

we have
\[
V_\mu^{\pi}-V_\mu^{\pi_{k}} 
\;\le\;\frac{1}{\eta} \text{KL}_{\delta_\mu^{\pi}}(\pi, \pi_k) - \frac{1}{\eta} \text{KL}_{\delta_\mu^{\pi}}(\pi, \pi_{k+1})
+ C_\pi (V_{(\delta_\mu^{\pi})'}^{\pi_{k}}-V_{(\delta_\mu^{\pi})'}^{\pi_{k}})+2C_\pi(\vartheta^\pi_\mu+1)\epsilon.
\]
and summing over \(k\) gives
\[
\sum_{i=0}^{k}\big(V_\mu^{\pi}- V_\mu^{\pi_{i}}  \big)
\;\le\;
\frac{1}{\eta} \text{KL}_{\delta_\mu^{\pi}}(\pi, \pi_0)  - \frac{1}{\eta} \text{KL}_{\delta_\mu^{\pi}}(\pi, \pi_{k}) 
+ C_\pi(V_{(\delta_\mu^{\pi})'}^{\pi_{k+1}}-V_{(\delta_\mu^{\pi})'}^{\pi_{0}})+2(k+1) C_\pi(\vartheta^\pi_\mu+1)\epsilon.
\]

By Lemma \ref{inexact:mon}, we conclude that
\[
V_\mu^\pi-V_\mu^{\pi_{k}}  
\;\le\; \frac{1}{k+1}\!\left( \frac{\text{KL}_{\delta_\mu^{\pi}}(\pi, \pi_0)}{\eta}
+ 2C_\pi\infn{V^\star_{+}} \right)+2 C_\pi(\vartheta^\pi_\mu+1) \epsilon+ \frac{\sum^{k-1}_{i=0}(i+1)C_{\pi_i}}{k+1}\epsilon.
\]

Second, with adaptive step size, for the first term,
\begin{align*}
&\sum_{s \in \cS}\delta^\pi_\mu(s)\sum_{a\in \cA} \Tilde{Q}^{\pi_k}(s,a)(\pi_{k}(a\,|\,s) - \pi_{k+1}(a\,|\,s))
 \\
&\ge \vartheta^\pi_\mu \sum_{s \in \cS}\delta^{\pi_{k+1}}_\mu(s) \sum_{a\in \cA} \Tilde{Q}^{\pi_k}(s,a)(\pi_{k}(a\,|\,s) - \pi_{k+1}(a\,|\,s)) \\
&\ge \vartheta^\pi_\mu \sum_{s \in \cS}\delta^{\pi_{k+1}}_\mu(s) \sum_{a\in \cA} Q^{\pi_k}(s,a)(\pi_{k}(a\,|\,s) - \pi_{k+1}(a\,|\,s)) \\&+  \vartheta^\pi_\mu \sum_{s \in \cS}\delta^{\pi_{k+1}}_\mu(s) \sum_{a\in \cA} (\Tilde{Q}^{\pi_k}(s,a)-Q^{\pi_k}(s,a))(\pi_{k}(a\,|\,s) - \pi_{k+1}(a\,|\,s))
\\&\ge \vartheta^\pi_\mu (V_\mu^{\pi_{k}} - V_\mu^{\pi_{k+1}})   - 2\vartheta^\pi_\mu C_{\pi_{k+1}} \epsilon
\end{align*}
where the last inequality follows from the Corollary \ref{cor:pdl} and Assumption \ref{assump:inexact}. Thus,
we have 
\begin{equation*}
\vartheta^\pi_\mu(U_{k+1} - U_k) + U_k
\le
\frac{1}{\eta_k} \text{KL}_{\delta^\pi_\mu}(\pi, \pi_k)
-
\frac{1}{\eta_k} \text{KL}_{\delta^\pi_\mu}(\pi, \pi_{k+1})
+
2 (C_{\pi_{k+1}}\vartheta^\pi_\mu+C_\pi)\epsilon.
\end{equation*}
where $U_k = V_\mu^\pi-V_\mu^{\pi_k} $. Dividing both sides by $\vartheta^\pi_\mu$ and rearranging terms yields
\begin{equation*}
U_{k+1}
+
\frac{1}{\eta_k \vartheta^\pi_\mu} \text{KL}_{\delta^\pi_\mu}(\pi, \pi_{k+1})
\le
\left(1 - \frac{1}{\vartheta^\pi_\mu}\right)
\left(
U_k
+
\frac{1}{\eta_k(\vartheta^\pi_\mu - 1)} \text{KL}_{\delta^\pi_\mu}(\pi, \pi_{k})
\right)
+
4 C_{k+1}\epsilon.
\end{equation*}
where $C_{k}= \max_{0\le i \le k}\{C_{\pi_{i}},C_\pi\}$. Since the step sizes satisfy
$\eta_{k+1}(\vartheta^\pi_\mu - 1) \ge \eta_k \vartheta^\pi_\mu$,
\begin{equation*}
U_{k+1}
+
\frac{1}{\eta_{k+1} (\vartheta^\pi_\mu-1)}  \text{KL}_{\delta^\pi_\mu}(\pi, \pi_{k+1})
\le
\left(1 - \frac{1}{\vartheta^\pi_\mu}\right)
\left(
U_k
+
\frac{1}{\eta_k(\vartheta^\pi_\mu - 1)}  \text{KL}_{\delta^\pi_\mu}(\pi, \pi_{k})
\right)
+
4C_{k+1}\epsilon.
\end{equation*}
Finally, Fact \ref{fact:seq} with
 $a_k = U_k + \frac{1}{\eta_k(\vartheta^\pi_\mu - 1)} \text{KL}_{\delta^\pi_\mu}(\pi, \pi_{k}),
\alpha = 1 - \frac{1}{\vartheta^\pi_\mu},
b = 4 C_k \epsilon,
$ leads to
\[
U_k
\le
\left(1 - \frac{1}{\vartheta^\pi_\mu}\right)^k
\left(
U_0
+
\frac{1}{\eta_0(\vartheta^\pi_\mu - 1)} \text{KL}_{\delta^\pi_\mu}(\pi, \pi_{0})
\right)
+
4\vartheta^\pi_\mu C_k\epsilon .
\]

\end{proof}
Now we are ready to prove Theorem \ref{thm:SNPG}.
\begin{proof}
 First, consider constant step size with $K= 2\left( \frac{\text{KL}_{\delta_\mu^{\pi}}(\pi, \pi_0)}{\eta}
+ 2C_\pi\infn{V^\star_{+}} \right)/{\epsilon'}$. By Theorem  \ref{thm:PE} and union bound over $0\le k \le K-1$, we have $\infn{Q^{\pi_k}-\hat{Q}^{\pi_k} } \le   \frac{\epsilon'}{2(4 \vartheta^\pi_\mu C_\pi  +KC_\pi)}$  with sample complexity 
 \begin{align*}
    KHN|\cS||\cA|=
& \mathcal{O}\left(\left(\frac{ t_{\frac{1}{2}}^3 (\vartheta^\pi_\mu)^2 R^2 C_\pi^5\left( \frac{\text{KL}_{\delta_\mu^{\pi}}(\pi, \pi_0)}{\eta}\right)^3
\infn{V^\star_{+}}^3 |\cS||\cA|}{(\epsilon')^5}  \right) \log^3 \left( 4C^2_K R(4 \vartheta^\pi_\mu   +K)/\epsilon' \right)\log\left(\frac{2K|\cS||\cA|}{\delta}\right)\right)\end{align*} 
for all $0 \le k \le K-1$.
Then, by Theorem \ref{thm:inexact},  we have $V_\mu^\pi-V_\mu^{\pi_k}\le \frac{\epsilon}{2}+\frac{\epsilon}{2} \le \epsilon.$ 

For adaptive step size $\eta_{k+1}(\vartheta^\pi_\mu  - 1) \ge \eta_k\vartheta^\pi_\mu $ with $K = \log \left(2\left(
V_\mu^\pi-V_\mu^{\pi_0}
+
\frac{\text{KL}_{\delta^\pi_\mu}(\pi, \pi_{0})}{\eta_0(\vartheta^\pi_\mu - 1)} \right)/\epsilon
\right) $, By Theorem  \ref{thm:PE}, $\infn{Q^\pi-\hat{Q}^\pi } \le   \frac{\epsilon'}{8 \vartheta^\pi_\mu C_k}$  with sample complexity
\[KNH|\cS||\cA|=\mathcal{O}\left(\frac{t_{\frac{1}{2}}^3 (\vartheta^\pi_\mu)^2C_K^2 R^2 |\cS||\cA|}{(\epsilon')^2}  \log^3 \left( 2C R\vartheta^\pi_\mu\infn{\delta^\pi_\mu}/\epsilon' \right) K\log\left(\frac{2K|\cS||\cA|}{\delta}\right)\right)\]
for all $0 \le k \le K-1$ .
Then, by Theorem \ref{thm:inexact},  we have $V_\mu^\pi-V_\mu^{\pi_k}\le \frac{\epsilon}{2}+\frac{\epsilon}{2} \le \epsilon.$ 
\end{proof}

\subsection{Proof of Corollary \ref{cor:sngd}}\label{appen:Cor}

\begin{proof}
In the previous proof of Theorem \ref{thm:inexact}, let $\pi =\pi^\star$, and use Corollary \ref{cor:pdl_reward} instead of Corollary \ref{cor:pdl} and the fact that $V^\pi_{\mu} \ge 0$. Then, we obtain the desired result.
\end{proof}

\section{Environments and additional experiment}\label{appen:experiments}
\subsection{Environments}
\paragraph{Frozenlake}
The Frozenlake environment is a \(4 \times 4\) grid world consisting of a goal state, three terminal states, and frozen states. The agent has four actions: \(\text{UP}(0)\), \(\text{RIGHT}(1)\), \(\text{DOWN}(2)\), and \(\text{LEFT}(3)\). If the agent is in a frozen state, the environment executes the left/forward/right variants of the intended action with probabilities \(1/3\) each. The agent receives a reward of \(1\) only if it reaches the goal state. If the agent attempts to move off the grid, it stays in place.

\paragraph{Cliffwalk} The Cliffwalk is a $3 \times 7$ grid world. The bottom right corner is the terminal goal state, and the states in the third row, except for the two end states, are terminal states. The agent has four actions: UP (0), RIGHT (1), DOWN (2), and LEFT (3). The MDP is deterministic, and the agent receives a reward of $1$ only when it reaches the goal state.  If the agent attempts to move off the grid, it stays in place.

\subsection{Experiment on Pathological MDP}

We run the projected policy gradient algorithm with $\alpha \in \{0.1, 0.05, 0.01\}$ and the natural policy gradient algorithm with both constant and adaptive step sizes. All algorithms are implemented using the transient policy gradient with the transient visitation measure. For Pathological MDP in Figure \ref{fig:example}, we use $\{0.1\cdot 1.01^k\}^\infty_{k=0}$  for the adaptive step size of natural policy gradient, where $k$ is the number of iterations, and $0.1$ for others.

The results are shown in Figure~\ref{fig:pathology}. As the figure shows, the policy error \(V^\star_{\mu}-V^{\pi_k}_\mu\) remains strictly positive due to a discontinuity at optimal policy. \(V^\star(s_1)=0>-1=V^\pi(s_1)\) for any nonoptimal policy \(\pi\) as discussed in Section~\ref{s:pathology}. Thus, the iterates produced by policy gradient methods do not converge to the optimal value \(V^\star_\mu\). 

The natural policy gradient with adaptive step size still exhibits the fastest convergence rate among the algorithms, as the guaranteed linear rate of Theorem~\ref{thm:ngd_linear} predicts. Note that both natural policy gradients converge to $V^\star_{+,\mu}$ while the projected policy gradient converges to $V_{\mu}^{\pi^\star_\alpha}$ for each $\alpha$, and smaller $\alpha$ makes projected policy gradient converge closer to $V_{+,\mu}^{\star}$ since $V_\mu^{\pi^\star_\alpha}$ increases monotonically to $V_{+,\mu}^\star$ as $\alpha \to 0$.

\section{Impact Statement}
Our work focuses on the theoretical aspects of reinforcement learning. There are no negative social
impacts that we anticipate from our theoretical results.

\begin{figure}
    \centering
\includegraphics[width=0.6\linewidth]{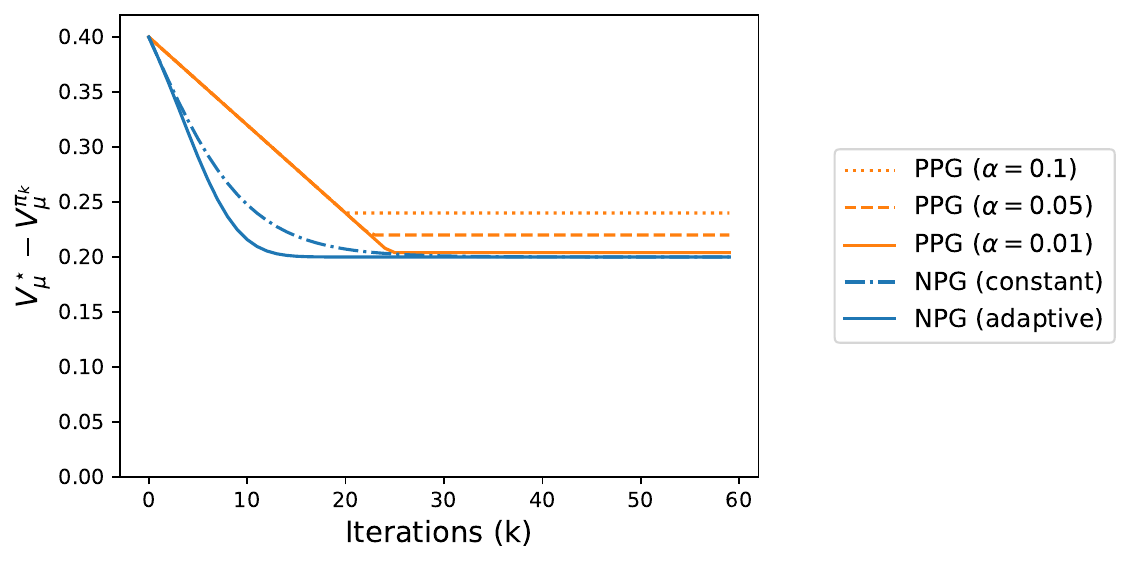}
    \caption{Comparison of projected policy gradient (PPG) and natural policy gradient (NPG) algorithms in Pathological MDP. Due to a discontinuity at optimal policy, $V^\star_\mu-V_\mu^{\pi_{k}} \ge V^\star_\mu - V^\star_{+,\mu} >0$. 
    }
    \label{fig:pathology}
\end{figure}


\end{document}